\documentclass[twoside,11pt]{article}
\pdfoutput=1
\usepackage{blindtext}

\usepackage[nohyperref, abbrvbib, preprint]{jmlr2e}

\usepackage{amsmath}
\usepackage{algorithm} 
\usepackage{algpseudocode}
\usepackage{mathtools}
\usepackage{ifthen}
\usepackage{thmtools,thm-restate}
\usepackage{macros-arxiv}
\usepackage{environments}
\DeclareMathAlphabet{\mathdutchcal}{U}{dutchcal}{m}{n}
\usepackage{xcolor}
\definecolor{dgreen}{rgb}{0.00,0.49,0.00}
\definecolor{dblue}{rgb}{0,0.08,0.75}
\usepackage{hyperref}
\hypersetup{
colorlinks,
linkcolor=dgreen,
citecolor=dblue,
urlcolor=black
}
\usepackage{cleveref}
\crefformat{footnote}{#2\footnotemark[#1]#3}

\jmlrheading{}{}{}{}{}{}{}

\ShortHeadings{Information Capacity Regret Bounds}{Eldowa, Cesa-Bianchi, Metelli, and Restelli}
\firstpageno{1}

\begin{document}

\title{Information Capacity Regret Bounds for \\Bandits with Mediator Feedback}

\author{\name Khaled Eldowa \email khaled.eldowa@unimi.it\\
       \addr Università degli Studi di Milano 
       \\Milano, 20133, Italy
       \AND
       \name Nicolò Cesa-Bianchi \email nicolo.cesa-bianchi@unimi.it \\
       \addr 
       Università degli Studi di Milano and Politecnico di Milano
       \\ Milano, 20133, Italy
       \AND
       \name Alberto Maria Metelli \email albertomaria.metelli@polimi.it \\
       \addr Politecnico di Milano 
       \\Milano, 20133, Italy
       \AND
       \name Marcello Restelli \email marcello.restelli@polimi.it \\
       \addr Politecnico di Milano 
       \\Milano, 20133, Italy
       }

\editor{My editor}

\maketitle

\begin{abstract}
This work addresses the mediator feedback problem, a bandit game where the decision set consists of a number of policies, each associated with a probability distribution over a common space of outcomes. Upon choosing a policy, the learner observes an outcome sampled from its distribution and incurs the loss assigned to this outcome in the present round. We introduce the policy set capacity as an information-theoretic measure for the complexity of the policy set. Adopting the classical EXP4 algorithm, we provide new regret bounds depending on the policy set capacity in both the adversarial and the stochastic settings. For a selection of policy set families, we prove nearly-matching lower bounds, scaling similarly with the capacity. We also consider the case when the policies' distributions can vary between rounds, thus addressing the related bandits with expert advice problem, which we improve upon its prior results. Additionally, we prove a lower bound showing that exploiting the similarity between the policies is not possible in general under linear bandit feedback. Finally, for a full-information variant, we provide a regret bound scaling with the information radius of the policy set.
\end{abstract}


\section{Introduction}
The framework of multi-armed bandits (MAB) models sequential decision-making problems with partial feedback. 
Real-world applications of this framework span a wide array of domains and include problems such as dynamic pricing \citep{misra2019dynamic} and advert placement \citep{schwartz2017customer}. 
In the classical non-stochastic MAB problem \citep{casino}, a learner, faced with a fixed
set of actions (also referred to as ``arms''), repeatedly interacts with the environment in a series of rounds by selecting an action 
and subsequently observing a numerical loss 
assigned beforehand to this action.
The performance of the learner is measured via the notion of regret, which compares the cumulative loss of the learner with that of the best action in hindsight.
The minimax regret for this problem, that is, the smallest achievable regret in the worst-case, is known to be of order $\sqrt{KT}$ \citep{audibert2009minimax}, with $K$ being the number of actions and $T$ the number of rounds. 

This formulation, however, fails to model situations in which, aside from observing the loss of the 
played action,
the learner could---in the same round---obtain information concerning the losses of other 
actions.
This information leakage could be a result of prior knowledge of
an underlying structure for the losses,
or owing to more explicit side observations.
Regardless of form, such extra information can lead to more efficient learning in the face of large (or even infinite) 
action sets.
A prominent example of structured losses is exhibited by the (adversarial) linear bandits problem, in which the 
action
set is a subset of $\R^d$ and the loss assigned to 
an action
in a given round is the inner product between the
action
and a common latent loss vector associated with that round. For this setting, \cite{bubeck2012towards} provide an algorithm achieving nearly optimal regret bounds of order $\sqrt{d T \log K}$ for finite 
action 
sets and $d\sqrt{T \log T}$ for compact actions sets. 
On the other hand, a simple form of side observations is modelled by the framework of online learning with graph feedback, in which, upon choosing 
an action,
the learner additionally observes the losses of the 
actions
adjacent to the chosen one in a given graph.
Algorithms exploiting this extra feedback can enjoy improved regret guarantees depending on the structure of the graph \citep[see, e.g.,][]{Alon15}.

In this work, we study a certain bandit model where the information leakage results from a combination of side observations and a structured assignment of losses.
In the basic template of this model, the learner is faced with a policy set whose elements are each associated with a probability distribution over a common (finite) space of outcomes.
At each round, a (latent) loss map associates each outcome with a numerical loss. Subsequently, the loss assigned to a policy is the loss of an outcome sampled from the policy's distribution.
Upon choosing a policy, the learner observes \emph{both} the sampled outcome and its loss. Depending on the problem, the loss map could be fixed over the rounds or changing in a stochastic or adversarial manner.  
The regret in this framework is defined as the difference between the (expected) cumulative loss of the learner and that of the optimal policy in hindsight.  
Our main goal is to understand how the structure of the policy set, in particular, how the similarity between the policies affects the achievable regret. 
One aspect of this problem reminiscent of linear bandits is that having chosen a policy, 
the learner's expected loss is linear in the policy's distribution, seen as a vector in the simplex. The distinction, however, is that the learner does not observe this quantity, as would be the case under linear bandit feedback; 
instead, one sampled outcome and its assigned loss are observed. 

This framework has been studied in the works of \cite{papini2019optimistic} and \cite{metelli2021policy}, where it was referred to as the \emph{mediator feedback} model.
The name here highlights the role of the outcomes as an extra layer of feedback ``mediating'' between the chosen policy and the observed loss, thereby allowing additional information gain regarding other policies.
The study of this feedback model was motivated in these works by a concrete problem in the context of reinforcement learning.
An instance of this problem is characterized by a Markov Decision Process (MDP) and a set of policies that map states to distributions over actions.
At every round of interaction, 
the learner selects a policy through which they interact with the MDP for a fixed horizon.
Naturally, the learner observes both the sampled trajectory and the accumulated reward, and aims to 
compete with the best policy from the given policy set. In this case, the trajectories are the outcomes over which each policy induces a probability distribution.

The feedback structure of the mediator framework is shared with the more classical problem of \emph{bandits with expert advice} \citep{casino}.
This problem is a variation of the (non-stochastic) MAB problem described above, where at the beginning of every round, the learner receives ``advice'' from each of a number of ``experts'' in the form of a probability distribution over the actions. The goal then becomes competing with the (expected) cumulative loss of the best expert in hindsight. Here, the actions of the MAB instance play the role of the outcomes, which are sampled from the distributions provided by the experts. Exactly fitting this problem into the mediator feedback framework additionally requires restricting the learner to only access the actions by way of sampling from (a mixture of) the experts' distributions, though this requirement is already satisfied by most state-of-the-art approaches. A more important distinction of the expert advice problem is the incorporation of a contextual element
in that the distributions recommended by the experts can vary from round to round.
Given our stated goal of studying the extent to which the similarity between the available distributions can be exploited, the addition of this contextual element is somewhat orthogonal to the main focus of this work, though it will still be briefly treated.

\subsection{Prior Results}
The \textsc{Exp4} algorithm was proposed in the work of \cite{casino} to address the bandits with expert advice problem, and remains an important benchmark in the contextual bandits literature. It was shown to enjoy a regret bound of order $\sqrt{KT \log N}$, where $N$ is the number of experts and $K$ still denotes the number of actions. As the recommendations of an expert can be seen as a strategy against which the learner is competing, this result shows that one can achieve a regret scaling only logarithmically with the number of such strategies. 
This was later shown to be nearly optimal by \cite{expertslowerbound}, who proved a lower bound of order $\sqrt{K T \log N / \log K}$. However, this lower bound concerns an instance where the number of experts is exponential in the number of actions, and the experts' distributions are deterministic.
The former is not surprising; if the number of experts is small compared to the arms, one can always achieve a regret of order $\sqrt{N T}$ by playing a minimax optimal bandit algorithm directly over the experts, entirely casting aside the structure of the problem. More importantly, the bound of \cite{casino} does not reflect one's expectation that the problem should become easier if the recommended distributions are more similar, and the lower bound does not address this question.

Via an elaborate modification of \textsc{Exp4}, \cite{McMahanS09} did address these very two points. For a fixed set of expert recommendations, their algorithm achieves a bound of order $\sqrt{\bigS T \log N}$, where $\bigS$, formally defined in \Cref{sec:capacity}, is a notion of effective size of the set of recommended distributions (i.e., the policy set). 
It satisfies $\bigS \leq \min\{K,N\}$, but can be smaller depending on the similarity between the distributions---or the ``agreement between the experts''. 
Specifically, it reaches its smallest value, $\bigS = 1$, when all the distributions are identical. 
One issue with this bound is that when the number of experts is small (say, only two), the fact that $\bigS \geq 1$ means that no substantial improvement is achieved over the $\sqrt{N T}$ bound, no matter how similar the distributions are.
Indeed, $\bigS - 1$ is arguably a more apt metric as it can shrink arbitrarily if the distributions are similar enough. In particular, it reduces to the total variation distance when there are only two distributions.
Nevertheless, even if the bound were to scale with this quantity, one may ask whether this is the best achievable dependence on the structure of the 
policy set.
We note in passing that the bound of \cite{McMahanS09} can also be achieved by plain \textsc{Exp4} in the general case \citep[see][Theorem 18.3]{lattimore2020bandit}, where it takes the form $\sqrt{\sum_t \bigS_t \log N}$ with $\bigS_t$ measuring 
the (dis)agreement between the experts at the $t$-th round.

While these results concern the adversarial regime, where no statistical constraints are placed on the losses, similar mediator feedback problems have been studied in the stochastic regime, where the losses are drawn at every round from a fixed distribution.
This includes the aforementioned works of \cite{papini2019optimistic} and \cite{metelli2021policy}, and that of \cite{stochastic-experts}, who consider a stochastic variant of the expert advice problem. Unlike the worst-case flavour (in terms of the dependence of the loss map) of the bounds in the adversarial regime, the 
results in these works are generally instance-based; the bounds enjoy a logarithmic dependence on the time horizon, but degrade in harder instances where suboptimal policies are difficult to discern.
Still, the dependence of these bounds on the policy set structure is largely independent of the loss map as it is primarily represented via diameter-like quantities measuring the pairwise maximum ``distance'' between the distributions according to some dissimilarity measure, mainly the chi-squared divergence or related quantities. Comparing 
$\bigS - 1$
with this chi-squared ``diameter'', neither quantity uniformly dominates the other, though the former is always bounded, while the latter need not be. 

In the online learning and bandits literature, best-of-both-worlds (BOBW) algorithms \citep{bubeck12bobw} address the adversarial and stochastic regimes simultaneously without prior knowledge of the nature of the environment. 
They guarantee regret (poly) logarithmic in the time horizon when the faced environment is stochastic, while retaining sub-linear regret against general environments. 
Of particular relevance to our setting is the recent work of \cite{bobw-blackbox}, where
they obtain the first BOBW guarantees for bandits with expert advice (or contextual bandits) using \textsc{Exp4} as a black-box decision rule within a cascade of two meta-algorithms. 
For stochastic environments, the bound is of order $K \log T \log N / \Delta$, where $\Delta$ denotes the minimum sub-optimality gap for the experts, whereas the traditional bound of $\sqrt{K T \log N}$ is guaranteed for all environments. 
As apparent, these bounds feature the usual coarse dependence on the number of actions and are therefore unable to reflect the affinity of the recommended distributions. 

\subsection[Contributions]{Contributions\footnote{A preliminary version of this work appeared as \citep{itw-paper}.}}
In this work, we consider a generic mediator feedback setting with finite policy and outcome sets.
We propose a new complexity (or effective size) measure for the policy set, which we refer to as the chi-squared capacity of the policy set, or simply the policy set capacity. 
This quantity (defined in \Cref{sec:capacity}) can be interpreted as the information capacity of a certain hypothetical communication channel induced by the policy set, 
albeit we define the mutual information between the input and output of said channel in terms of the chi-squared divergence in lieu of the standard Kullback–Leibler divergence. The channel referred to here is
one whose input alphabet is the policy set and output alphabet is the outcome set, 
while its transition matrix is defined such that conditioned on an input policy, the output is drawn from the policy's distribution.
Besides its more natural information-theoretic interpretation, this notion of capacity is never larger than either $\bigS - 1$ or the maximum pairwise chi-squared divergence. 
Moreover, when the policy set consists of only two policies, the capacity reduces to a (symmetric) divergence measure of the same order as the squared Hellinger distance and the triangular discrimination. 

In \Cref{sec:capacity}, we consider the adversarial regime and provide an improved regret bound of order
$\max\{ \sqrt{\capc T \log N}, \log N \}$ for \textsc{Exp4}, where $\capc$ denotes the capacity. Unlike prior results, the horizon-dependent term in this bound can shrink arbitrarily if the distributions are similar enough. 
We then extend this result to the case when the policies' distributions can vary between rounds, thus addressing the general bandits with expert advice problem.
In particular, we show  
that the same algorithm run with an adaptive learning rate enjoys a bound
essentially  
of order $\sqrt{\sum_t \capc_t \log N}$, where $\capc_t$ is the capacity of the policies' distributions at round $t$. 
This bound is obtained as an implication of a stronger history-dependent bound, where the complexity of the policy set in a given round is represented through the mutual (chi-squared-)information between the chosen policy and the drawn outcome, conditioned on the events up to the previous round.

Still considering the \textsc{Exp4} algorithm, 
we provide best-of-both-worlds bounds in \Cref{sec:bobw}. 
In the stochastic regime, we show that the algorithm enjoys a bound of order $\capc \log T \log (NT)/\Delta$, where $\Delta$ is the minimum sub-optimality gap for the policies.
Simultaneously, the algorithm is shown to retain a worst-case bound of order $\sqrt{\capc T \log T \log N}$. 
The former bound follows from a more general guarantee that we provide for the adversarially corrupted stochastic regime, an intermediate regime commonly considered in BOBW works \citep[see, e.g.,][]{ito-bobw-graphs,bobw-blackbox}. 
The proof of 
this result builds upon the techniques developed in \citep{ito-bobw-graphs} for proving BOBW bounds for a related algorithm in the setting of online learning with graph feedback.

We complement these results in \Cref{sec:lower} by proving worst-case lower bounds for three families of policy sets. 
These lower bounds scale with the policy set capacity in the same manner as the regret bound we provided for \textsc{Exp4} in the adversarial setting,
asserting its optimality up to factors logarithmic in the number of policies. 
Additionally, in \Cref{sec:linear}, we prove another lower bound through which we aim to compare the studied feedback model with that of linear bandits. 
As alluded to before, the policies' distributions can be treated as a set of arms belonging to the simplex under an alternative linear bandit formulation of the problem, where the learner observes the inner product between the chosen policy and the latent loss map. 
Considering a particular family of policy sets, we prove that under linear bandit feedback, one must incur regret of order at least $\sqrt{N T}$ against any policy set in this family, even as the capacity approaches zero. 
This shows that the attainability of regret guarantees that improve as policies become more similar is a distinctive feature of mediator feedback.

Finally, in \Cref{sec:full-info}, we consider a full-information variant of the problem, where the learner observes the entire loss map at every round. 
This can be seen as a generalization of the prediction with expert advice problem \citep{cesa1997use}.
For this setting, we show that a simple Online Mirror Descent strategy enjoys a regret bound of order $\sqrt{\capckl T}$, where $\capckl$ is an altered form of the policy set capacity, based on the more standard KL-divergence.
This quantity is no larger than $\log N$, and can be interpreted as an information radius of the policy set.

\subsection{Additional Related Works} \label{sec:related}
\cite{krishnamurthy2020contextual} study a contextual bandit problem with continuous actions, where the learner competes with a set of competitor policies mapping states (contexts) to actions.  
Instead of placing a smoothness assumption on the loss function, they opt for minimizing a notion of smoothed regret.
More precisely, they fix a smoothing kernel that maps each action to a distribution over action. Accordingly, they obtain a new competitor class of smoothed policies that map states to distributions over actions by composing the original policies with the smoothing kernel. This has the effect of forcing a favourable structure on the policy set. 
Indeed, they obtain (via \textsc{Exp4}) a bound of order $\sqrt{\kappa T \log N}$ in the adversarial regime, where $N$ is the number of policies and $\kappa$, called the kernel complexity, is defined as the largest possible density assigned by the smoothing kernel with respect to some base probability measure. This latter quantity upper bounds the continuous analogue of $\bigS$ for any given context. 
Similar smoothed benchmarks are studied in \citep{majzoubi2020efficient} and \citep{zhu2022contextual} under realizability assumptions. 

In its dependence on the mutual information between policies and outcomes, 
one of the regret bound we provide for \textsc{Exp4} (see \Cref{thm:chi-general})
bears some superficial resemblance to the PAC-Bayesian results in \citep{seldin2011pac}. 
For a stochastic contextual bandits problem, \cite{seldin2011pac} prove bounds on the per-round instantaneous regret that depend on the mutual information between the observed state and the chosen action. 
This quantifies the complexity of the adopted decision rule, which is traded off against its empirical regret measured according to past observations. 
Another notable mention is the information-theoretic analysis of \cite{russo2016information} for Thompson sampling in Bayesian bandit problems, which results in bounds scaling with the mutual information between the faced environment and the optimal action, or some satisficing benchmark \citep{russo2022satisficing,arumugam2021deciding}. 
This measures the amount of information that needs to be acquired about the environment to identify the target action. 

Another related line of work concerns the best arm identification (BAI) problem, a variant of the MAB problem where the learner's aim is to find the optimal arm efficiently.
\cite{reddy2023} and \cite{poiani2023pure} study the BAI problem with the added constraint that the learner can only sample arms via a number of given stochastic policies.
As the objective remains identifying the optimal arm, the manner in which the structure of the policy set affects the achievable regret is fundamentally different from our setting.
Indeed, in their problem, a more diverse policy set can be advantageous to the learner.

\section{Preliminaries}
In this section, we start by reviewing some concepts from information theory that will be referenced throughout the rest of the paper.
Then, we lay down a formal statement of the main problem setting.

\subsection{Information Theory Background} \label{sec:info-theory}
Let $f:(0,\infty) \to \R$ be a convex function with $f(1)=0$, 
and define the limits $f(0) = \lim_{x \to 0^+}f(x)$ and $f'(\infty) = \lim_{x \to \infty}f(x)/x$ (either of which could be infinite). 
If $P$ and $Q$ are two distributions (probability mass functions) on a common finite set $\Omega$, the $f$-divergence (\citealp{Ali1966AGC}; \citealp{Csiszar}; \citealp[Section 7.1]{itbook}) between them is defined as:
\begin{equation*}
    D_f(P \,\|\, Q) \coloneqq \sum_{x \in \Omega} Q(x) f\bigg(\frac{P(x)}{Q(x)}\bigg) \,,
\end{equation*}
with the understanding that $0f(0/0)=0$ and $0f(a/0)=\lim_{x \to 0^+}xf(a/x)=af'(\infty)$ for $a>0$. Notable properties of $f$-divergences include joint convexity in $P$ and $Q$, non-negativity, and the fact that $D_f(P \,\|\, P) = 0$.  
Examples for $f$-divergences used in this work include
\begin{itemize}
    \item $f(x)=(1/2)|x-1| \longrightarrow$ total variation distance:
    \begin{equation*}
        \delta(P,Q) \coloneqq \frac{1}{2} \summ_x |P(x) - Q(x)| = 1 - \summ_x \min\{P(x), Q(x)\} \,.
    \end{equation*}
    \item $f(x)=(1/2)(\sqrt{x}-1)^2 \longrightarrow$ squared Hellinger distance:
    \begin{equation*}
        H^2(P,Q) \coloneqq \frac{1}{2} \summ_x (\sqrt{P(x)} - \sqrt{Q(x)})^2 \,.
    \end{equation*}
    \item $f(x)=(x-1)^2/(x+1) \longrightarrow$ triangular discrimination (also known as the Vincze–Le Cam divergence \citep{sason2016f}):
    \begin{equation*}
        \Delta(P,Q) \coloneqq \summ_x \frac{(P(x)-Q(x))^2}{P(x)+Q(x)} \,.
    \end{equation*}
    \item $f(x)=x \ln x \longrightarrow$ KL-divergence:
    \begin{equation*}
        \kl{P}{Q} \coloneqq \summ_x P(x) \log\lrb{\frac{P(x)}{Q(x)}} \,.
    \end{equation*}
    \item $f(x)=(x-1)^2 \longrightarrow$ chi-squared divergence:
    \begin{equation*}
        \chisqr{P}{Q} \coloneqq \summ_x Q(x) \bigg(\frac{P(x)}{Q(x)}-1\bigg)^2 = \summ_x \frac{P(x)^2}{Q(x)} - 1 \,.
    \end{equation*}
\end{itemize}

Let $X$ and $Y$ be two discrete random variables mapping $\Omega$ to the finite sets $\mathcal{X}$ and $\mathcal{Y}$ respectively. The quantity defined as
\begin{equation*}
    D_f(P_{Y|X} \,\|\, Q_{Y|X} \,|\, P_X) \coloneqq \sum_{x \in \mathcal{X}} P_X(x) D_f(P_{Y|{X=x}} \,\|\, Q_{Y|{X=x}}) 
\end{equation*}
is known as the conditional $f$-divergence, where a summand corresponding to some $x\in \Xs$ is set to zero if $P_X(x)=0$.
An immediate property of $f$-divergences is that if $P_{X,Y} = P_X P_{Y|X}$ and $Q_{X,Y} = P_X Q_{Y|X}$, then
\begin{equation} \label{eq:conditional-divergence-prop}
    D_f(P_{X,Y} \,\|\, Q_{X,Y}) = D_f(P_{Y|X} \,\|\, Q_{Y|X} \,|\, P_X) \,.
\end{equation}
When jointly distributed according to $P_{X,Y}$, the mutual information between $X$ and $Y$ based on a given $f$-divergence (or their mutual $f$-information) is defined as \citep[Section 7.8]{itbook}:
\begin{equation*}
    I_f(X;Y) \coloneqq D_f({P_{X,Y}} \,\|\, {P_X P_Y}) \,,
\end{equation*}
which is the divergence between their joint distribution and the product of the marginals; hence, it is non-negative and symmetric in $X$ and $Y$. Moreover, when $f$ is strictly convex at unity (as is the case for the above examples), we have that $I_f(X;Y) = 0$ if and only if $X$ and $Y$ are independent \citep{makur2020comparison}.
Via \eqref{eq:conditional-divergence-prop}, it holds that\footnote{By symmetry, this also holds when the roles of $X$ and $Y$ are exchanged.}
\begin{equation*}
    I_f(X;Y) = D_f(P_{Y|X} \,\|\, P_Y \,|\, P_X) = \summ_x P_X(x) D_f(P_{Y|{X=x}} \,\|\, \textstyle{\summ_{x'}} P_X(x') P_{Y|{X=x'}}) \,.
\end{equation*} 
The previous identity justifies the overloaded notion $I_f(P_X, P_{Y|X}) \coloneqq I_f(X;Y)$ formulating $I_f$ as a function of $P_X$ and the kernel $P_{Y|X}$.

When the $f$-divergence of choice is the KL-divergence, we obtain the standard mutual information, denoted simply as $I(X;Y)$, whereas $\Ic(X;Y)$ will denote the mutual information based on the chi-squared divergence. 
The latter is bounded from above by $\min\{|\Xs|,|\mathcal{Y}|\}-1$ considering that
\begin{align*}
    \Ic(X;Y)
    &= \sum_{x \colon P_X(x) > 0} P_X(x) \sum_{y \colon P_{Y|X=x}(y) > 0} P_{Y|X=x}(y)  \frac{P_{Y|X=x}(y)}{\sum_{x' \in \Xs} P_X(x') P_{Y|X=x'}(y)} - 1 \\
    &\leq \sum_{x \colon P_X(x) > 0} P_X(x) \sum_{y \colon P_{Y|X=x}(y) > 0} P_{Y|X=x}(y)  \frac{P_{Y|X=x}(y)}{ P_X(x) P_{Y|X=x}(y)} - 1 
    \leq |\Xs| - 1 \,,
\end{align*}
which holds with equality if the distributions $\{P_{Y|X=x}\}_{x \in \Xs}$ have disjoint supports and $P_X$ has full support.
Symmetrically, we also have that
$\Ic(X;Y) \leq |\mathcal{Y}| - 1$.
On the other hand, as $\kl{P}{Q} \leq \log(\chisqr{P}{Q} + 1)$, we have that $I(X;Y) \leq \min\{\log|\Xs|,\log|\mathcal{Y}|\}$.
In particular, $I(X;Y) = \log|\Xs|$ holds if the distributions $\{P_{Y|X=x}\}_{x \in \Xs}$ have disjoint supports and $P_X$ is uniform. Another distinction between the two quantities concerns their behaviour as functions of $P_X$. While for a fixed kernel $P_{Y|X}$, $I(P_X, P_{Y|X})$ is continuous in $P_X$ \citep[Section 7.3]{cover2006}, the same does not hold in general for $\Ic$. To see this, consider a simple instance where $\Xs = \mathcal{Y} = \{0,1\}$, $P_{Y|X=0}(0)=0.5$, and $P_{Y|X=1}(0)=1$. If $P_X(0) = \epsilon$ for some $\epsilon \in (0,1]$, then
\[    \Ic(P_X, P_{Y|X}) = \frac{2 - (3/2)\epsilon}{2 - \epsilon} - \frac{1}{2} \eqqcolon g(\epsilon) \,, \]
which satisfies $\lim_{\epsilon \to 0^+} g(\epsilon) = 1/2$ even though $\Ic(P_X, P_{Y|X}) = 0$ at $\epsilon = 0$ by definition 
since $X$ and $Y$ become independent.
Moreover, since $g$ is decreasing in the interval $(0,1]$, 
$\Ic$ as a function of $P_X$ attains no maximum.

Maximizing the standard mutual information $I(P_X, P_{Y|X})$ in $P_X$ gives rise to what we will refer to as the (KL-)information capacity of the kernel $P_{Y|X}$, denoted as:
\begin{equation*}
    \capckl(P_{Y|X}) \coloneqq \max_{P_X \in \mP_{\Xs}} I(P_X, P_{Y|X}) \,,
\end{equation*}
where $\mP_{\Xs}$ is the set of possible distributions over the elements of $\Xs$. 
More practically, $\capckl(P_{Y|X})$ is better known as the information capacity of the discrete memoryless stationary channel (DMC) with input alphabet $\Xs$, output alphabet $\mathcal{Y}$, and transition matrix $P_{Y|X}$ 
(\citealp[Chapter 7]{cover2006}; \citealp[Chapter 19]{itbook}).
This quantity has an operational significance as it quantifies the highest rate per channel use at which information can be reliably sent (\citealp[Theorem 7.7.1]{cover2006}; \citealp[Theorem 19.9]{itbook}).
Analogously, we define the $\chi^2$-capacity of $P_{Y|X}$ as:
\begin{equation*}
    \capcchi(P_{Y|X}) \coloneqq \sup_{P_X \in \mP_{\Xs}} \Ic(P_X, P_{Y|X}) \,,
\end{equation*}
where we use the supremum in place of the maximum as the latter might not exist per the counterexample provided earlier.

\subsection{Problem Setting}
Let $\aspace \coloneqq \{1,\dots,\noutcomes\}$ denote a set of $\noutcomes \geq 2$ outcomes, 
and let $\Theta \coloneqq \{\theta_1,\dots,\theta_N\} \subset \Delta_\noutcomes$ denote a policy set consisting of $N \geq 2$ distributions over the outcomes, where $\Delta_\noutcomes$ is the probability simplex in $\R^\noutcomes$ defined as $\{u \in \R^\noutcomes \:\colon\: \sum_{j=1}^\noutcomes u(j) = 1 \:\text{and}\: u(j) \geq 0 \: \forall j \in [\noutcomes] \}$. Hence, for an outcome $\oc \in \aspace$ and policy $\theta \in \Theta$, $\theta(x)$ denotes the probability assigned to $\oc$ by $\theta$.
We consider a mediator feedback problem, where a learner plays a sequential game with an unknown environment for $T$ rounds.
From the environment's characteristic distribution, a latent sequence of loss vectors $(\ell_t)_{t=1}^T$ is drawn at the beginning of the game, where $\ell_t \in [0,1]^\noutcomes$ maps each outcome to a loss at the $t$-th round.
Ensuingly, the learner sequentially interacts with the environment by selecting at each round $t$ a policy $\vartheta_t \in \Theta$, possibly at random, and subsequently observing the pair $(\Oc_t,\ell_t(\Oc_t))$, where $\Oc_t$ is a random outcome distributed according to $\vartheta_t$.
Slightly overloading the notation, we let $\ell_t(\theta)$ denote the expected value (conditioned on $\ell_t$) of $\ell_t(\Oc_t)$ had the learner picked policy $\theta$ at round $t$; that is, $\ell_t(\theta) \coloneqq \sum_{\oc \in \aspace} \theta(\oc) \ell_t(\oc)$.
The learner's objective is to minimize their regret, defined as:
\begin{equation*}
    R_T \coloneqq \E \lsb{\sum_{t=1}^T \ell_t(\vartheta_t)} - \min_{\theta \in \Theta} \E \lsb{\sum_{t=1}^T \ell_t(\theta)} \,,
\end{equation*}
where the expectation is taken over both the learner's and the environment's randomization.
We use a common probability space $(\Omega, \mathcal{F},\pr)$ to define all random variables.
For round $t \in [T]$, let  $\his_t \coloneqq (\vartheta_s, \Oc_s, \ell_s(\Oc_s))_{s=1}^t$ denote the interaction history up to the end of round $t$, and let $\mathcal{F}_t \coloneqq \sigma(\his_t)$ denote the $\sigma$-algebra generated by $\his_t$.
Accordingly, we define $\E_t[\cdot] \coloneqq \E[\cdot \mid \mathcal{F}_{t-1}]$ and $\pr^t(\cdot) \coloneqq \pr(\cdot \mid \mathcal{F}_{t-1})$, with $\mathcal{F}_{0}$ being the trivial $\sigma$-algebra. Analogously to \citep{russo2016information}, we define $I^t(X;Y)$ and $\Ic^t(X;Y)$ as the mutual information and the mutual chi-squared-information between (discrete) random variables $X$ and $Y$ with $\pr^t$ as the base measure. Notice that these quantities are random variables owing to their dependence on the history.

\section{The Policy Set Capacity: An Improved Regret Bound for EXP4} \label{sec:capacity}
Before defining the policy set capacity, we provide some context by briefly reviewing some quantities used in related works to
describe the richness of the policy set.
\cite{McMahanS09} introduce the 
quantity
\begin{equation*}
    \bigS(\Theta) \coloneqq \sum_{\oc \in \aspace} \max_{\theta \in \Theta} \theta(\oc) \,. 
\end{equation*}
It is easily verified that $1 \leq \bigS(\Theta) \leq \min\{\noutcomes,N\}$, where the lower bound is attained in the limit case when all the policies are identical,
and the upper bound is attained either when the policies have disjoint supports or when each outcome is matched with a policy entirely concentrated on that outcome. To get a finer sense of this quantity, we define $\V(\Theta) \coloneqq \bigS(\Theta) - 1$. Notice then that when there are only two policies, that is, $\Theta = \{ \theta_1, \theta_2\}$, $\V$ reduces to the total variation distance between the two distributions: 
\begin{align*}
    \V(\{ \theta_1, \theta_2\}) = \summ_{\oc} \max\{\theta_1(\oc), \theta_2(\oc)\} - 1 = 1 - \summ_{\oc} \min\{\theta_1(\oc), \theta_2(\oc)\} = \delta(\theta_1, \theta_2) \,. 
\end{align*}
More generally, one can obtain the somewhat coarse bound:
\begin{align*}
    \V(\Theta) \leq \min_{\median \in \Delta_{\noutcomes}} \summ_\theta \delta(\theta, \median) \,,
\end{align*}
where the distribution minimizing the right-hand side acts as the geometric median of $\Theta$ in terms of the total variation distance. This inequality follows from Theorem II.1 in \citep{minimax-risk}, which for any $f$-divergence and any $\median \in \Delta_{\noutcomes}$,
provides the (implicit) bound:
\begin{align*}
    f(\bigS(\Theta)) + (N-1) f\bbrb{\frac{N-\bigS(\Theta)}{N-1}} \leq \summ_\theta D_f(\theta \,\|\, \median) \,.
\end{align*}

Another relevant quantity is the chi-squared ``diameter'' of the policy set:
\begin{align*}
    d_{\chi^2}(\Theta) \coloneqq \max_{\theta,\theta' \in \Theta} \chisqr{\theta}{\theta'} \,,
\end{align*}
which is featured in the regret bounds of \cite{stochastic-experts} and \cite{papini2019optimistic}, see \Cref{sec:bobw}. Though it has no general upper bound, $d_{\chi^2}$ can be smaller than $\V$ as shown in the examples section below. \cite{stochastic-experts} also obtain bounds in terms of another diameter-like quantity based on the logarithm of (one plus) the $f$-divergence with $f(x)=x\exp(x-1) - 1$, though $d_{\chi^2}$ is never larger. 

\subsection{The Policy Set Capacity}
Let $\vartheta$ and $\Oc$ be two random variables taking values respectively over $\Theta$ and $\aspace$ such that $\pr_{\Oc|\vartheta=\theta}(\oc) = \theta(\oc)$ for any $\theta \in \Theta$ and $\oc \in \aspace$. Then, we define the (chi-squared) capacity of the policy set as:
\begin{align*}
    \capc(\Theta) \coloneqq \capcchi(\pr_{\Oc|\vartheta}) \,,
\end{align*}
which does not depend on the distribution of $\vartheta$. More explicitly, if we define
\begin{align*}
    \Q_\tau(\Theta) \coloneqq \Ic(\tau, \pr_{\Oc|\vartheta}) = \summ_{\theta} \tau(\theta) \bchisqr{\theta} {\textstyle{\summ_{\theta'}} \tau(\theta')\theta'} \,,
\end{align*}
for some distribution $\tau \in \mP_{\Theta}$;\footnote{Inline with previous notation, $\mP_{\Theta}$ denotes the set of possible distributions over the policies.} then, $\capc(\Theta) = \sup_{\tau \in \mP_{\Theta}} \Q_\tau(\Theta)$. As alluded to before, this definition inspires an interpretation of the policy set as inducing a stationary, memoryless  
channel defined via the kernel $\pr_{\Oc|\vartheta}$. 
Intuitively, $\capc(\Theta)$ can be seen to measure the dependency between $\vartheta$ and $X$ 
maximised over the prior distribution of $\vartheta$.
Hence, the more dissimilar the distributions are, the larger this quantity.

Since $\capc$ is based on $\Ic$, it satisfies $0 \leq \capc(\Theta) \leq \min\{\noutcomes,N\} - 1$, which is the same range as that of $\V$.
In particular, much like $\V$, the upper bound is attained either when the policies have disjoint supports or when each outcome is matched with a policy entirely concentrated on that outcome.
Moreover, $\capc(\Theta) = 0$ if and only if $X$ and $\vartheta$ are independent no matter how $\vartheta$ is distributed, which requires the policies to be identical.
More distinctively, it holds in general that $\capc(\Theta) \leq \min \{\V(\Theta), d_\Chisqr(\Theta)\}$.
On the one hand, the (joint) convexity of the chi-squared divergence implies that
\begin{align*} 
    \capc(\Theta) 
    &\leq \sup_{\tau \in \mP_{\Theta}} \sum\nolimits_{\theta, \theta'} \tau(\theta) \tau(\theta') \chisqr{\theta} {\theta'} \leq \max_{\theta, \theta'} \chisqr{\theta} {\theta'} = d_\Chisqr(\Theta)\,.
\end{align*}
On the other hand, we have that
\begin{align*}
   \capc(\Theta) 
   &= \sup_{\tau \in \mP_{\Theta}} \summ_{\theta} \tau(\theta) \bigg(\summ_{\oc} \frac{\theta(\oc)^2}{\sum_{\theta'} \tau(\theta') \theta'(\oc)} - 1 \bigg) \\
   &= \sup_{\tau \in \mP_{\Theta}} \summ_{\oc} \frac{\sum_{\theta} \tau(\theta)  \theta(\oc)^2}{\sum_{\theta'} \tau(\theta') \theta'(\oc)} - 1 \\
   &\leq \sup_{\tau \in \mP_{\Theta}} \summ_{\oc} \max_{\theta''} \theta''(\oc) \frac{\sum_{\theta} \tau(\theta)  \theta(\oc)}{\sum_{\theta'} \tau(\theta') \theta'(\oc)} - 1 \\&
   =  \summ_{\oc} \max_{\theta} \theta(\oc) - 1 
   = \bigS(\Theta) - 1 = \V(\Theta)\,.
\end{align*}

\subsection{A Regret Bound for EXP4 in Terms of the Capacity}
\textsc{Exp4}, detailed in \Cref{alg:exp4}, adopts a simple and natural approach for tackling mediator feedback problems.
Its choice of policy in a given round is drawn from a running distribution over the policies taking an exponential weights form.
There, each policy $\theta$ is weighted according to a proxy of the sum of its losses so far, where an importance-weighted estimator $\hat{\ell}_t(\theta)$ 
replaces the inaccessible
$\ell_t(\theta)$.
The following theorem provides a regret bound for \textsc{Exp4} that scales with the policy set capacity.
This result improves upon the $\sqrt{\bigS(\Theta) T \log N}$ bound,
seemingly the best available worst-case bound for the considered setting.
Further, we instantiate the capacity in the ensuing discussion for three families of policy sets for which the bound of this theorem will be shown to be near-optimal in \Cref{sec:lower}. 
While the proposed learning rate schedule requires exact knowledge of the capacity, this requirement will be lifted in \Cref{thm:chi-general}, which also addresses the case when the policies' distributions can vary between rounds. 

\begin{algorithm} [t]
    \caption{\textsc{Exp4} (Fixed Policy Set)}
    \label{alg:exp4}
    \begin{algorithmic}[1]
        \State \textbf{Input:}  sequence of learning rates $(\eta_t)_{t=1}^T$
        \State \textbf{Initialize:} $\forall \theta \in \Theta$, $\hat{\ell}_0(\theta)=0$
        \For{$t=1,\dotsc,T$}
            \State Draw $\vartheta_t \sim p_t$, where $p_{t}(\theta) = \frac{\exp(-\eta_t \sum_{s=0}^{t-1}\hat{\ell}_s(\theta))}{\sum_{\theta'}  \exp(-\eta_t \sum_{s=0}^{t-1}\hat{\ell}_s(\theta'))}$
            \State Draw $\Oc_t \sim \vartheta_t$, and observe loss $\ell_t (\Oc_t)$
            \State $\forall \theta \in \Theta$, set $\hat{\ell}_t(\theta) = \frac{\theta(\Oc_t)}{\sum_{\theta'} p_{t} (\theta') \theta'(\Oc_t)} \ell_t(\Oc_t)$
        \EndFor
    \end{algorithmic}
\end{algorithm}

\begin{theorem} \label{thm:chi}
    Algorithm \ref{alg:exp4} with $\eta_t = \min\Bcb{1,\sqrt{ \frac{\log N}{e \capc(\Theta) t }}}$ satisfies
    \[
        R_T \leq 2\max\bcb{\sqrt{e \capc(\Theta) T \log N}, \log N} \,.
    \]
\end{theorem}
\begin{proof}
    Let $\theta^* \in \argmin_{\theta \in \Theta} \E\sum_{t=1}^T \ell_t(\theta)$. For a policy $\theta$, we define a shifted version of the loss at time $t$ as 
    $\zeta_t (\theta) \coloneqq \summ_\oc \brb{\theta(\oc) - \psi_t(\oc)} \ell_t(\oc)$,
    where $\psi_t(\oc) \coloneqq \summ_{\theta} p_{t} (\theta) \theta(\oc)$. Thus, $\zeta_t (\theta) = \ell_t(\theta) - \summ_{\theta'} p_{t} (\theta') \ell_t(\theta')$.
    Notice that for any two policies $\theta$ and $\theta'$, $\zeta_t(\theta) - \zeta_t(\theta') = \ell_t(\theta) - \ell_t(\theta')$. 
    Hence,
    $
        R_T = \E \summ_t \brb{\ell_t(\vartheta_t) - \ell_t(\theta^*)} = \E \summ_t \brb{\zeta_t(\vartheta_t) - \zeta_t(\theta^*)} \,.
    $
    Next, we define $\hat{\zeta}_t (\theta)$ as an estimate of the shifted loss of $\theta$ at time $t$:
    \begin{align} \label{def:estimated-modified-losses}
        \hat{\zeta}_t (\theta) &\coloneqq \brb{\theta(\Oc_t)-\psi_t(\Oc_t)} \frac{\ell_t (\Oc_t)}{\psi_t (\Oc_t)} 
        = \hat{\ell}_t(\theta) - \ell_t(\Oc_t) \,.
    \end{align} 
    For convenience, we will sometimes treat the distribution $p_t$ as a vector belonging to the simplex $\Delta_N \subset \R^N$, where its $i$-th coordinate $p_t(i)$ denotes $p_t(\theta_i)$ for each $i \in [N]$.
    Analogously, the functions $\hat{\ell}_t$ and $\hat{\zeta}_t$ will sometimes be handled as vectors in $\R^N$.
    Notice that $p_t$ and $\psi_t$ are measurable with respect to $\F_{t-1}$, and that $\ell_t$ is independent of $\vartheta_t$ and $\Oc_t$ conditioned on $\F_{t-1}$.  
    Hence, it holds that $\E_t\zeta_t(\vartheta_t) = \E_t \summ_{\theta} p_{t} (\theta) \zeta_t(\theta)$, and that $\E_t \hat{\zeta}_t(\theta) = \E_t \zeta_t(\theta)$ for any fixed $\theta \in \Theta$.
    Consequently, thanks to the tower rule and the linearity of expectation, we have that 
    \begin{align*}
        \E \summ_t \brb{\zeta_t(\vartheta_t) - \zeta_t(\theta^*)} = \E \summ_t \langle p_t - \mathbf{e}_{\theta^*}, \zeta_t \rangle = \E \summ_t \langle p_t - \mathbf{e}_{\theta^*}, \hat{\zeta}_t \rangle \,,
    \end{align*}
    where $\mathbf{e}_{\theta^*} \in \R^N$ is the indicator vector for $\theta^*$.

    It is well known \citep[see][Section~2.7]{shalev2012online} that for every round $t$, the definition of $p_t$ in \Cref{alg:exp4} is equivalent to
    \begin{equation} \label{eq:ftrl-form-1}
    p_t = \argmin_{p \in \Delta_N} \: \eta_t \Ban{ \summ_{s=1}^{t-1}\hat{\ell}_s, p } - H(p) \,,
    \end{equation}
    where $H(p) \coloneqq \sum_{i=1}^N p(i)\log(1/p(i))$ is the Shannon entropy of $p$.
    Note that for any $p \in \Delta_N$, 
    \[
     \Ban{ \summ_{s=1}^{t-1}\hat{\zeta}_s, p } =  \Ban{ \summ_{s=1}^{t-1}\hat{\ell}_s , p } -  \summ_{s=1}^{t-1} \ell_s(\Oc_s) \,.
    \]
    Hence, by adding constant terms (i.e., not depending on $p$) to the objective function in \eqref{eq:ftrl-form-1} and changing the scaling, we can arrive at the following alternative characterization of $p_t$ for $t \in [T+1]$:
    \begin{equation*} 
    p_t = \argmin_{p \in \Delta_N} \: \Ban{ \summ_{s=1}^{t-1}\hat{\zeta}_s, p } + \frac{1}{\eta_t} \brb{\log N-H(p)} \,, 
    \end{equation*}
    which is equivalent to the update rule of the follow the regularized leader (FTRL) algorithm when executed on the losses $(\hat{\zeta}_t)_{t \in [T]}$ with a decision set $\Delta_N$ and a sequence of regularizers $(\phi_t)_{t \in [T+1]}$  where
    \[
    \phi_t(p) = \frac{1}{\eta_t} \brb{\log N-H(p)} \qquad \forall p \in \Delta_N \,, 
    \]
    which is the negative Shannon entropy normalized to the range $[0,\log N]$ and scaled by the learning rate.
    Let $D_{\phi_t}(\cdot\,;\,\cdot)$ be the Bregman divergence based on $\phi_t$, and set $\eta_{T+1}=\eta_T$.
    We can then use Lemma 7.14 in \citep{orabona2023modern} to obtain the following regret bound for FTRL on the estimated shifted losses:
    \begin{equation*}
        \summ_t \langle p_t - \mathbf{e}_{\theta^*}, \hat{\zeta}_t \rangle \leq \frac{\log N}{\eta_T} + \frac12 \summ_t \eta_t \langle z_t,  \hat{\zeta}_t^{\:2} \rangle \,,
    \end{equation*}
    where $z_t$ lies on the line segment between $p_t$ and $\tilde{p}_{t+1}=\argmin_{u \in \R^N_{\geq 0}} \langle \hat{\zeta}_t, u \rangle + D_{\phi_t}(u\,;\,p_t)$. 
    By its definition, it is easy to show that $\tilde{p}_{t+1}(i) = p_t(i)\exp(-\eta_t \hat{\zeta}_t(i))$ for every $i \in [N]$. 
    Notice that
    $
        \eta_t \hat{\zeta}_t(i) 
        \geq - \eta_t \ell_t(\Oc_t)  
        \geq - \eta_t \geq -1 
    $
    since $\hat{\ell}_t$ is non-negative, $\eta_t \in (0,1]$, and $\ell_t(\Oc_t) \leq 1$.
    Hence, it holds for every $i \in [N]$ that $\tilde{p}_{t+1}(i) \leq e \: p_t(i)$, implying that
    $\langle z_t,  \hat{\zeta}_t^{\:2} \rangle \leq e \langle p_t,  \hat{\zeta}_t^{\:2} \rangle \,.$
    
    Overall, we have shown that
    \begin{equation} \label{eq:regret-pre-q-form}
        R_T \leq \frac{\log N}{\eta_T} + \frac{e}{2} \summ_t \eta_t \E \summ_\theta p_t(\theta) \hat{\zeta}_t(\theta)^2 \,.
    \end{equation}
    Now, for every $\theta \in \Theta$ and $t \in [T]$, we have that
    \begin{align*}
        \E_t \hat{\zeta}_t(\theta)^2 &= \E_t \brb{\theta(\Oc_t)-\psi_t(\Oc_t)}^2 \frac{\ell_t (\Oc_t)^2}{\psi_t (\Oc_t)^2} \\
        &\leq \E_t \frac{\brb{\theta(\Oc_t)-\psi_t(\Oc_t)}^2}{\psi_t (\Oc_t)^2} \\
        &= \E_t \summ_\oc \frac{\brb{\theta(\oc)-\psi_t(\oc)}^2}{\psi_t (\oc)^2} \I\{\oc=\Oc_t\} \\
        &= \summ_\oc \frac{\brb{\theta(\oc)-\psi_t(\oc)}^2}{\psi_t (\oc)} \\
        &= \summ_\oc \psi_t (\oc) \bbrb{\frac{\theta(\oc)}{\psi_t (\oc)} - 1}^2 
        = \chisqr{\theta}{\psi_t} = \bchisqr{\theta} {\textstyle{\summ_{\theta'}} p_t(\theta')\theta'}\,,
    \end{align*}
    where the third equality holds since $\E_t \I\{\oc=\Oc_t\} = \pr^t(\oc=\Oc_t) = \psi_t (\oc)$.
    This implies that
    \begin{align} \label{eq:second-moment-q-bound}
        \E_t \summ_\theta p_t(\theta) \hat{\zeta}_t(\theta)^2 \leq \summ_\theta p_t(\theta)  \bchisqr{\theta} {\textstyle{\summ_{\theta'}} p_t(\theta')\theta'} = \Q_{p_t}(\Theta)\,.
    \end{align}
    Consequently, we arrive at the following bound:
    \begin{align*}
         R_T &\leq \frac{\log N}{\eta_T} + \frac{e}{2} \E \summ_t \eta_t \Q_{p_t}(\Theta) 
         \leq \frac{\log N}{\eta_T} + \frac{e}{2} \capc(\Theta) \summ_t \eta_t \,.
    \end{align*}
    If $T \geq \frac{\log N}{e \capc(\Theta)}$, then $\eta_T = \sqrt{\frac{\log N}{e \capc(\Theta) T }}$ and 
    $
        R_T \leq 2\sqrt{e \capc(\Theta) T \log N}\,,
    $
    where we have used that $\eta_t \leq \sqrt{\frac{\log N}{e \capc(\Theta) t }}$ for $t \in [T]$ and that $\summ_{t=1}^T \frac{1}{\sqrt{t}} \leq 2\sqrt{T}$. Otherwise, if $T < \frac{\log N}{e \capc(\Theta)}$, then $\eta_1=\dots=\eta_T=1$ and 
    $
        R_T \leq \log N + \frac{e \capc(\Theta) T}{2} \leq 2 \log N \,.
    $
\end{proof}

\subsection{Examples} \label{sec:examples}
We now examine the quantity $\capc(\Theta)$ for a selection of policy set structures and compare it with related quantities.

\subsubsection{Two Policies} \label{sec:examples:two}
We start with the case when the policy set consists of only two policies, i.e., $\Theta = \{\theta_1,\theta_2\}$.
As mentioned before, we have that $\V(\Theta) = \tv(\theta_1,\theta_2)$, while $d_{\Chisqr} = \max\{ \chisqr{\theta_1}{\theta_2}, \chisqr{\theta_2}{\theta_1} \}$.
These two quantities are incomparable in general, and this can be seen by specializing the next example to the two policies case.
For a fixed $r \in [0,1]$, define $q_r(\theta_1 \,\|\, \theta_2) = \Q_\tau(\Theta)$ with $\tau(\theta_1)= r$; hence, $\capc(\Theta) = \sup_{r \in [0,1]} q_r(\theta_1 \,\|\, \theta_2) \eqqcolon \capc(\theta_1,\theta_2)$. An explicit form for $q_r(\theta_1 \,\|\, \theta_2)$ is given by:
\begin{align*}
   q_r(\theta_1 \,\|\, \theta_2) = r (1-r) \summ_\oc \frac{ (\theta_1(\oc)  - \theta_2(\oc))^2}{r\theta_1(\oc) + (1-r)\theta_2(\oc)} \,.
\end{align*}
As a function of $r \in [0,1]$, $q_r(\theta_1 \,\|\, \theta_2)$ is concave with $q_0(\theta_1 \,\|\, \theta_2)=q_1(\theta_1 \,\|\, \theta_2)=0$.
This quantity is known in the literature as the Vincze–Le Cam divergence of order $r$ \citep{raginsky2016strong,makur2020comparison},
which is an $f$-divergence with $f(x)=\frac{r(1-r)(x-1)^2}{r(x-1)+1}$.
Corresponding to $r=1/2$ is (half) the triangular discrimination $\Delta$, immediately implying a lower bound for $\capc$: 
\begin{align*}
    \capc(\theta_1,\theta_2) &\geq  q_{1/2}(\theta_1 \,\|\, \theta_2) = \frac{1}{2} \summ_\oc \frac{(\theta_1(\oc)-\theta_2(\oc))^2}{\theta_1(\oc)+\theta_2(\oc)} = \frac{1}{2} \Delta(\theta_1,\theta_2) \,.
\end{align*}
On the other hand, since $\frac{r(1-r)(x-1)^2}{r(x-1)+1} \leq (\sqrt{x}-1)^2$ for any $r \in (0,1)$ and $x\in [0,\infty)$, we can bound $\capc$ in terms of the squared Hellinger distance $H^2$:
\begin{equation} \label{eq:capacity-hellinger}
    \capc(\theta_1,\theta_2) \leq \summ_\oc \theta_2(\oc) \lrb{\sqrt{{\theta_1(\oc)}/{\theta_2(\oc)}}-1}^2 = 2 H^2(\theta_1,\theta_2)\,.
\end{equation}
Combining these observations with known inequalities \citep{topsoe}, we obtain that
\begin{equation*}
    \tv^2 \leq \frac\Delta2 \leq \capc \leq 2H^2 \leq \Delta \leq 2 \tv \,,
\end{equation*}
which shows that the capacity of two policies is of the same order as the squared Hellinger distance and the triangular discrimination. For the two policies case, we prove in \Cref{thm:lower:two} a lower bound of $\Omega(\sqrt{\capc(\theta_1,\theta_2) T})$, which order-wise matches the bound of \Cref{thm:chi}.

\subsubsection{\texorpdfstring{$\epsilon$}{Epsilon}-Greedy Policies} \label{sec:examples:epsilon}
Consider now a case in which $N=\noutcomes$ and each policy $\theta$ is associated (one-to-one) with an outcome $\oc_\theta$ such that for any outcome $\oc$, $\theta(\oc) = ({1-\epsilon})/{N} + \epsilon \mathbb{I}\{\oc=\oc_\theta\}$, where $\epsilon \in [0,1]$. At $\epsilon=0$, all policies collapse to the uniform distribution, and we get that $\capc(\Theta)=0$. On the other hand, when $\epsilon=1$, the problem essentially reduces to a standard bandit problem with $\capc(\Theta)=N-1$.
Generally, for $\tau \in \mP_{\Theta}$, $\Q_\tau(\Theta)$ takes the following form:
\begin{equation*}
    \Q_\tau(\Theta) = \epsilon^2 \summ_\theta \frac{\tau(\theta)(1-\tau(\theta))}{\frac{1-\epsilon}{N} + \epsilon \tau(\theta)} \,.
\end{equation*}
For intermediate values of $\epsilon \in (0,1)$, $\Q_\tau(\Theta)$ is a strictly concave function in $\tau$ attaining its maximum value at the uniform distribution,
entailing that $\capc(\Theta) = \epsilon^2 (N-1)$.
In comparison, we have that 
\begin{equation*}
  \V(\Theta) = \epsilon (N-1) \qquad \text{and} \qquad   d_{\Chisqr}(\Theta) = \frac{\epsilon(N-2) + 2}{\epsilon(N-1) + 1} \cdot \frac{\epsilon^2}{1-\epsilon} N \,.
\end{equation*}
Notice that even though $d_{\Chisqr}$ grows unbounded as $\epsilon$ approaches 1, it can be smaller than $\V$ for small enough $\epsilon$.
In \Cref{thm:lower:epsilon}, we prove a lower bound of order $\epsilon\sqrt{NT}$ for this policy set structure, which matches the upper bound of \Cref{thm:chi} up to a logarithmic factor.

\subsubsection{\texorpdfstring{$M$}{M}-Supported Uniform Policies} \label{sec:examples:uniform}
Consider another policy set structure where all policies are uniform distributions over a support of $M \leq \noutcomes$ outcomes. That is, if we denote by $\text{Supp}(\theta)$ the support for policy $\theta$, then we have that $\text{Supp}(\theta)=M$ and for any outcome $\oc$, $\theta(\oc)=({1}/{M})\I\{\oc \in \text{Supp}(\theta)\}$. 
Assume further that each outcome belongs to the support of at least one policy.
For this structure, one can verify that 
$\capc(\Theta) = \V(\Theta) = \noutcomes/M - 1$.
In fact, we have that $\Q_\tau(\Theta) = \capc(\Theta)$ for any $\tau \in \mP_{\Theta}$ with full support.
On the other hand, $d_{\Chisqr} = \infty$ outside of the trivial case when $M = \noutcomes$.
In \Cref{thm:lower:multi}, we show that for a special family of $M$-supported uniform policies (where $N \geq \noutcomes/M$), the regret of any algorithm is 
$\Omega\brb{{\sqrt{({\noutcomes/M - 1}) \:T\: {\log(N)}/{\log\lrb{{\noutcomes}/{M}}}}}}$. This lower bound particularly shows that the logarithmic factor in the regret bound of \Cref{thm:chi} is at least partly unavoidable.

\subsection{A Generalization for Time-Varying Policy Distributions}
The next theorem extends the result of \Cref{thm:chi} by allowing the distributions of the policies to vary between rounds, modelling in this manner the problem of bandits with expert advice.
We rely on an adaptive learning rate schedule, whose form is common in the online learning and bandits literature \citep{AUER-adaptive,McMahanS10,neu15}.
The resulting bound replaces the dependence on the (per-round) capacity with the history-conditioned mutual chi-squared-information between the chosen policy and the drawn outcome, recalling that the former is an upper bound for the latter by definition.
Additionally, the adopted learning rate in a given round only requires an upper bound on the capacity of the policies' distributions, thus affording one the flexibility of providing a quantity of simpler form like $\V$, or even just $\min\{N,\noutcomes\}$.
In terms of regret, this flexibility is paid for through a solitary additive term depending primarily on the largest of these provided bounds.

Only in the current scope, a member $\theta$ of the policy set $\Theta$ is not synonymous with a distribution over the outcomes; it serves solely as an identifier for a policy. 
Along with the sequence of losses, the environment draws for each $\theta \in \Theta$ a sequence of distributions $(\theta(\cdot;t))_{t=1}^T$ at the beginning of the game, where $\theta(\oc;t)$ is the probability assigned to outcome $\oc$ by policy $\theta$ at round $t$. 
The distributions $(\theta(\cdot;t))_{\theta \in \Theta}$ associated with a given round $t$ are revealed to the learner at the beginning of the round. 
Accordingly, we redefine the interaction history as ${\his}_t \coloneqq \brb{(\vartheta_s, \Oc_s, \ell_s(\Oc_s))_{s\in [t]}, (\theta(\cdot;s))_{\theta \in \Theta, s\in [t+1]}}$, which includes all the information available to the learner before choosing a policy at round $t+1$.\footnote{Consistently, the usage of the filtration $(\F_t)_t$ (and dependent quantities) in the context of the following result refers to this augmented definition of the history.}
The definition of the regret remains the same, only that the loss of policy $\theta$ at round $t$ is now defined as $\ell_t(\theta) \coloneqq \summ_\oc \theta(\oc;t) \ell_t(\oc)$.
For a distribution $\tau \in \mP_{\Theta}$ and round $t \in [T]$, we define 
\begin{align*}
    \Q_{t,\tau}\phantheta \coloneqq \summ_{\theta} \tau(\theta) \bchisqr{\theta(\cdot;t)} {\textstyle{\summ_{\theta'}} \tau(\theta')\theta'(\cdot;t)}
\end{align*}
and $\capc_t\phantheta \coloneqq \sup_{\tau \in \mP_{\Theta}} \Q_{t,\tau}\phantheta$ as the time-varying analogues of $Q_\tau$ and $\capc$. 
The following theorem still concerns the plain \textsc{EXP4} algorithm, reformulated in \Cref{alg:exp4-varying} for the time-varying case. Observe that for any round $t$, $\Ic^t(\vartheta_t; \Oc_t) = \Q_{t,p_t}\phantheta$.

\begin{algorithm} [t]
    \caption{\textsc{Exp4} (Time-Varying Policy Distributions)}
    \label{alg:exp4-varying}
    \begin{algorithmic}[1]
        \State \textbf{Input:}  sequence of learning rates $(\eta_t)_{t=1}^T$
        \State \textbf{Initialize:} $\forall \theta \in \Theta$, $\hat{\ell}_0(\theta)=0$
        \For{$t=1,\dotsc,T$}
            \State Observe distributions $\theta(\cdot ; t) \quad \forall \theta \in \Theta$  
            \State Draw $\vartheta_t \sim p_t$, where $p_{t}(\theta) = \frac{\exp(-\eta_t \sum_{s=0}^{t-1}\hat{\ell}_s(\theta))}{\sum_{\theta'}  \exp(-\eta_t \sum_{s=0}^{t-1}\hat{\ell}_s(\theta'))}$
            \State Draw $\Oc_t \sim \vartheta_t$, and observe loss $\ell_t (\Oc_t)$
            \State $\forall \theta \in \Theta$, set $\hat{\ell}_t(\theta) = \frac{\theta(\Oc_t;t)}{\sum_{\theta'} p_{t} (\theta') \theta'(\Oc_t;t)} \ell_t(\Oc_t)$
        \EndFor
    \end{algorithmic}
\end{algorithm} 

\begin{theorem} \label{thm:chi-general}
    Let $Z_t \coloneqq \sum_{s=1}^t \Ic^s(\vartheta_s; \Oc_s)$ for all $t \in [T]$, and let $(J_t)_{t=1}^T$ be a non-decreasing sequence of non-negative real numbers such that $J_t$ is $\mathcal{F}_{t-1}$-measurable and $J_t \geq \capc_t\phantheta$. 
    Then, \Cref{alg:exp4-varying} with $\eta_t = \sqrt{ \frac{\log N}{\log N + e \lrb{Z_{t-1} + J_t }}}$ satisfies
    \begin{align*}
        R_T &\leq \E \bbsb{ 2\sqrt{e \summ_t \Ic^t(\vartheta_t; \Oc_t) \log N } + \log N + \sqrt{e J_T \log N}} \\
        &\leq \E \bbsb{ 2\sqrt{e \summ_t \capc_t\phantheta \log N } + \log N + \sqrt{e J_T \log N}}\,.
    \end{align*}
\end{theorem}
\begin{proof}
    For any $\theta \in \Theta$ and $t \in [T]$, 
    let $\hat{\zeta}_t(\theta)$ be defined as in \eqref{def:estimated-modified-losses}.
    Notice that the distributions $(\theta(\cdot;t))_{\theta \in \Theta}$ are measurable with respect to $\F_{t-1}$.
    Moreover, the sequence $(\eta_t)_t$ is non-increasing and $\eta_t \leq 1$ holds for all rounds.
    Hence, with the same arguments laid out in the proof of \Cref{thm:chi}, one can show that
    \begin{equation*}
        R_T \leq \E \bbsb{ \frac{\log N}{\eta_T} + \frac{e}{2} \summ_t \eta_t \summ_\theta p_t(\theta) \hat{\zeta}_t(\theta)^2} \,.
    \end{equation*}
    Furthermore, similar to what was shown in the proof of \Cref{thm:chi}, it holds for every $t$ that $\E_t \summ_\theta p_t(\theta) \hat{\zeta}_t(\theta)^2 \leq \Q_{t,p_t}\phantheta = \Ic^t(\vartheta_t; \Oc_t)$.
    Hence, since $\eta_t$ is $\mathcal{F}_{t-1}$-measurable, we get that
    \begin{equation*}
        R_T \leq \E \bbsb{\frac{\log N}{\eta_T} + \frac{e}{2}  \summ_t \eta_t \Ic^t(\vartheta_t; \Oc_t)} \,.
    \end{equation*}
    We then conclude the proof by bounding the two terms inside the expectation. Starting with the second term, we have that
    \begin{align*}
        \frac{e}{2}  \summ_t \eta_t \Ic^t(\vartheta_t; \Oc_t) &= \frac{\sqrt{\log N}}{2} \summ_t \frac{e \Ic^t(\vartheta_t; \Oc_t)}{\sqrt{\log N + e \lrb{Z_{t-1} + J_t }}} \\&
        \leq \frac{\sqrt{\log N}}{2} \summ_t \frac{e \Ic^t(\vartheta_t; \Oc_t)}{\sqrt{e Z_t}} 
        \leq \sqrt{e Z_T \log N } \,,
    \end{align*}
    where the last inequality follows via Lemma 3.5 in \citep{AUER-adaptive}. Whereas
    \begin{align*}
        \frac{\log N}{\eta_T} 
        &= \sqrt{\log^2 N + e \lrb{Z_{T-1} + J_T } \log N} 
        \leq  \sqrt{e Z_T \log N } + \log N + \sqrt{e J_T \log N} \,.
    \end{align*}
\end{proof}
Let $\bigS_t\phantheta \coloneqq \summ_\oc \max_\theta \theta(\oc; t)$, and let $\V_t\phantheta \coloneqq \bigS_t\phantheta - 1$. A reasonable choice is to set $J_t = \max_{s \leq t} \V_s\phantheta$, which would only cause the bound to concede an added term of $\sqrt{e \max_{t \leq T} \V_t\phantheta \log N}$ while lifting the more burdensome requirement of computing the capacity at each round. 
Notice that the first bound of \Cref{thm:chi-general} depends (in expectation) on the observed sequence of losses through its dependence on the algorithm's decision at each round (i.e., $p_t$). However, it is unclear whether this bound can take advantage of any particular benign property of the losses when compared with the second bound, which only depends on the policies' distributions. Nevertheless, for what concerns the bandits with expert advice problem, these bounds improve upon the state of the art bound of $\sqrt{\summ_t \bigS_t \log N}$ reported in \citep[Theorem 18.3]{lattimore2020bandit}.

\section{Best-of-Both-Worlds Bounds} \label{sec:bobw}
Besides the adversarial regime considered thus far, we study in this section a more benign setting 
where the dependence of the regret on the time horizon can be improved.
Specifically, we will consider what we will refer to as
the adversarially corrupted stochastic regime,
where it is assumed that there exists a policy $\stotheta \in \Theta$ such that for every round $t$ and policy $\theta \neq \stotheta$,
\begin{equation*}
    \E \bsb{\ell_t(\theta) - \ell_t(\stotheta)} \geq \Delta - \corrupt_t \,,
\end{equation*}
for some 
$\Delta \in (0,1]$ and $\corrupt_t \geq 0$. 
Additionally, we define $\corrupt \coloneqq \sum_{t=1}^T \corrupt_t$. 
This includes, as a special case, the canonical stochastic regime where the loss functions $(\ell_t)_t$ are independently and identically distributed across rounds. 
Notice that besides the addition of corruption, the more general stochastic regime we consider does not require the losses to be distributed either stationarily or independently.
For similar setups, see, for example,  \citep{wei2018more,zimmert2021tsallis,ito-bobw-graphs,bobw-blackbox}.

The main result of this section concerns, once again, the \textsc{Exp4} algorithm, which we show to enjoy BOBW bounds when coupled with a certain learning rate schedule.
For the stochastic regime, the algorithm achieves a bound linear in the capacity and only poly-logarithmic in the time horizon.
Simultaneously, it retains roughly the same worst-case guarantee as that of \Cref{thm:chi}.
This result, provided in the next theorem, is obtained by combining elements from the proof of \Cref{thm:chi} with the learning rate schedule and analysis technique used by \cite{ito-bobw-graphs} 
in the setting of online learning with strongly observable feedback graphs.
Before stating the theorem, 
we define 
\begin{equation*}
    P(\theta) \coloneqq \sum_{t=1}^T (1-p_t(\theta)) \qquad \text{and} \qquad \overbar{P}(\theta) \coloneqq \E\sum_{t=1}^T (1-p_t(\theta))
\end{equation*}
for every policy $\theta \in \Theta$, where $p_t(\theta) = \pr^t(\vartheta_t = \theta)$ as specified in \Cref{alg:exp4}.
Moreover, for any distribution $p$ over the policies, we let $H(p) \coloneqq \summ_\theta p(\theta) \log({1}/{p(\theta)})$ denote its Shannon entropy as before.

\begin{theorem} \label{thm:bobw}
    Let $\gamma = \sqrt{\frac{e \capc(\Theta) \log(eT)}{2\log(N)}}$, and suppose \Cref{alg:exp4} is run with $\eta_t = \min \bcb{1,\frac{1}{\beta_t}}$ for $t \in [T+1]$, where $\beta_1 = \gamma$ and for $t \in [T]$,
    $
        \beta_{t+1} = \beta_t + \frac{\gamma}{\sqrt{1+\frac{1}{\log N}\summ_{s=1}^t H(p_t)}} \,.
    $
    Then, it holds in general that
    \[
        R_T \leq  3\sqrt{2e \capc(\Theta) T \log(eT) \log (eN)} + \log N \,,
    \]
    whereas in the adversarially corrupted stochastic regime, the algorithm additionally satisfies
    \begin{align*}
        R_T \leq 36 e^2 \frac{\capc(\Theta) \log(eT) \log (NT) }{\Delta} + 3 e \sqrt{\frac{2 \capc(\Theta) \log(eT) \log (NT)  \corrupt}{ \Delta}} +  2 \log N + 4 \Delta\,.
    \end{align*}
\end{theorem}
\begin{proof}
    Let $\theta^* \in \argmin_{\theta \in \Theta} \E \summ_t \ell_t(\theta)$, which need not coincide with $\stotheta$. For policy $\theta$ and round $t$, let $\hat{\zeta}_t(\theta)$ be defined as in \eqref{def:estimated-modified-losses}. 
    As shown in the proof of \Cref{thm:chi},\footnote{We at times treat $p_t$ and $\hat{\zeta}_t$ as vectors in $\R^N$ in the manner described before in the proof of \Cref{thm:chi}.} we have that $R_T = \E \summ_t \langle p_t - \mathbf{e}_{\theta^*}, \hat{\zeta}_t \rangle$,
    where $\mathbf{e}_{\theta^*} \in \R^N$ is the indicator vector for $\theta^*$. Also, similar to what was argued in that proof, \Cref{alg:exp4} produces its predictions according to the following FTRL rule:
    \begin{equation*} 
    p_t = \argmin_{p \in \Delta_N} \: \Ban{ \summ_{s=1}^{t-1}\hat{\zeta}_s, p } + \phi_t(p) \,, 
    \end{equation*}
    where for $p \in \Delta_N$ and $t \in [T+1]$, $\phi_t(p) = -\frac{1}{\eta_t}H(p)$. Note that the sequences $\beta_t$ and $\eta_t$ are increasing and non-increasing, respectively. Hence, we have that $\phi_t(p) - \phi_{t+1}(p) \geq 0$. With this in mind, one can extract the following bound from the proof of \Cref{thm:chi} and the proof of Lemma 7.14 in \citep{orabona2023modern}:
    \begin{align*}
        \summ_t \langle p_t - \mathbf{e}_{\theta^*}, \hat{\zeta}_t \rangle &\leq \phi_{T+1}(\mathbf{e}_{\theta^*}) - \phi_1(p_1) + \summ_t \brb{\phi_t(p_{t+1})-\phi_{t+1}(p_{t+1})} \\
        &\hspace{20em}+ \frac{e}{2} \summ_t \eta_t  \summ_\theta p_t(\theta) \hat{\zeta}_t(\theta)^2 \\
        &= \frac{1}{\eta_1} \log N + \summ_t \bbrb{\frac{1}{\eta_{t+1}} - \frac{1}{\eta_t}} H(p_{t+1}) + \frac{e}{2} \summ_t \eta_t \summ_\theta p_t(\theta) \hat{\zeta}_t(\theta)^2 \,.
    \end{align*}
    Since, $\eta_t$ is measurable with respect to $\F_{t-1}$, we have via \eqref{eq:second-moment-q-bound} that
    $
        \E_t \eta_t \summ_\theta p_t(\theta) \hat{\zeta}_t(\theta)^2 \leq \eta_t \Q_{p_t}(\Theta) \leq \eta_t \capc(\Theta) \,.
    $
    Consequently, it holds that
    \begin{equation*}
        R_T \leq \frac{1}{\eta_1} \log N + \E\summ_t \bbrb{\frac{1}{\eta_{t+1}} - \frac{1}{\eta_t}} H(p_{t+1}) + \frac{e}{2} \capc(\Theta) \E \summ_t \eta_t\,.
    \end{equation*}
    For every round $t$, we clearly have that $\eta_t \leq \frac{1}{\beta_t}$, and that $\frac{1}{\eta_t} = \max\{1,\beta_t\}$. Hence, since $\beta_{t+1} \geq \beta_t$, it holds that 
    $
    \frac{1}{\eta_{t+1}} - \frac{1}{\eta_t} = \brb{\max\{1,\beta_{t+1}\} - \max\{1,\beta_t\}} \leq \beta_{t+1} - \beta_{t} \,.
    $
    Consequently,
    \begin{equation} \label{bobw:eq:beta-regret}
        R_T \leq \max\{1,\gamma\} \log N + \E\summ_t \brb{\beta_{t+1} - \beta_t} H(p_{t+1}) + \frac{e}{2} \capc(\Theta) \E \summ_t \frac{1}{\beta_t} \,.
    \end{equation}
    The following two facts can be extracted from the proof of Proposition 1 in \citep{ito-bobw-graphs}:
    \begin{align*}
        \summ_t \brb{\beta_{t+1} - \beta_t} H(p_{t+1}) &\leq 2 \gamma \sqrt{\log N} \sqrt{\summ_t H(p_t)} \\
        \summ_t \frac{1}{\beta_t} &\leq \frac{\log(eT)}{\gamma \sqrt{\log N}} \sqrt{\log N + \summ_t H(p_t)} \,.
    \end{align*}
    Moreover, Lemma 4 in the same paper entails that
    $
        \summ_t H(p_t) \leq P(\stotheta) \log \frac{eNT}{P(\stotheta)} \,.
    $
    Plugging these inequalities back into \eqref{bobw:eq:beta-regret} yields that
    \begin{align} 
        R_T &\leq \max\{1,\gamma\} \log N + 2 \gamma \sqrt{\log N} \E \sqrt{P(\stotheta) \log \frac{eNT}{P(\stotheta)}} \nonumber \\ \label{bobw:eq:P-regret}
        &\hspace{16em}+ \frac{e  \capc(\Theta) \log(eT) }{2 \gamma \sqrt{\log N}} \E \sqrt{\log N + P(\stotheta) \log \frac{eNT}{P(\stotheta)}} \,.
    \end{align}
    To obtain the worst-case bound, simply observe that $u \log \frac{eNT}{u}$ is an increasing function in $u$ for $0 < u \leq NT$. Hence, 
    $
         P(\stotheta) \log \frac{eNT}{P(\stotheta)} \leq T \log (eN) \,,
    $
    which, together with \eqref{bobw:eq:P-regret}, implies that
    \begin{equation*}
        R_T \leq \max\{1,\gamma\} \log N + \bbrb{2 \gamma \sqrt{\log N} +  \frac{e  \capc(\Theta) \log(eT) }{\gamma \sqrt{\log N}}} \sqrt{T \log (eN)}  \,,
    \end{equation*}
    from which the desired bound can be seen to hold after plugging in the value of $\gamma$.
    
    Towards proving the second bound, we take an alternative route following the proof of Theorem~4 in \citep{ito-bobw-graphs}.
    Namely, we argue that if $P(\stotheta) \leq e$, then $P(\stotheta) \log \frac{eNT}{P(\stotheta)} \leq e \log (NT)$, otherwise, $P(\stotheta) \log \frac{eNT}{P(\stotheta)} \leq P(\stotheta) \log(NT)$. Resuming again from \eqref{bobw:eq:P-regret}, we get that
    \begin{align*}
        R_T &\leq \max\{1,\gamma\} \log N  + \bbrb{2 \gamma \sqrt{\log N} +  \frac{e  \capc(\Theta) \log(eT) }{\gamma \sqrt{\log N}}} \sqrt{\log(NT)} \E \sqrt{\max\{P(\stotheta),e\}} \\
        &\leq \max\bbcb{\log N, \sqrt{\frac{e}{2} \capc(\Theta) \log(eT)\log(N)}}  + 2\sqrt{2 e \capc(\Theta) \log(eT) \log (NT) }\lrb{\sqrt{\overbar{P}(\stotheta)} + \sqrt{e}} \\
        &\leq \log N  + 3e\sqrt{2 \capc(\Theta) \log(eT) \log (NT) }\lrb{1 + \sqrt{ \overbar{P}(\stotheta)}} \,,
    \end{align*}
    where the second inequality follows after plugging in the value of $\gamma$ and using Jensen's inequality. For what follows, we define $G_1 \coloneqq \log N$  and $G_2 \coloneqq 3e\sqrt{2 \capc(\Theta) \log(eT) \log (NT) }$. 
    
    Now, in the adversarially corrupted stochastic regime, we observe that
    \begin{align*}
        R_T \geq \summ_t \E \bsb{\ell_t(\theta) - \ell_t(\stotheta)} \geq \Delta \E \summ_t \I\{\vartheta_t \neq \stotheta\} - \corrupt 
        = \Delta \overbar{P}(\stotheta) - \corrupt \,.
    \end{align*}
    Combining this with the last bound, we obtain that for any $\lambda > 0$,
    \begin{align*}
        R_T &= (1+\lambda) R_T - \lambda R_T 
        \leq (1 + \lambda) \Brb{G_1 + G_2 + G_2 \sqrt{ \overbar{P}(\stotheta)}} - \lambda \Delta \overbar{P}(\stotheta) + \lambda \corrupt\,.
    \end{align*}
    Using the fact that $2a\sqrt{u} - b u \leq a^2 / b$ for all $u,a,b \geq 0$, we get that
    \begin{align*}
        R_T &\leq (1+\lambda) (G_1 + G_2) + \lambda \corrupt + \frac{(1+\lambda)^2 G_2^2}{4 \lambda \Delta} \\
        &= (1+\lambda) (G_1 + G_2) + \frac{G_2^2}{2 \Delta} + \lambda \bbrb{\corrupt + \frac{G_2^2}{4 \Delta} } + \frac{ G_2^2}{4 \lambda \Delta}.  
    \end{align*}
    If we choose $\lambda \coloneqq \sqrt{\frac{G_2^2}{4\Delta}\big/\brb{\corrupt + \frac{G_2^2}{4\Delta}}}$, we obtain the following bound: 
    \begin{align*}
        R_T &\leq 2 G_1 + 2 G_2 + \frac{G_2^2}{2 \Delta} + 2\sqrt{\frac{G_2^2 \corrupt}{4 \Delta}  + \frac{G_2^4}{16 \Delta^2} } 
        \leq 2 G_1 + 2 G_2 + \frac{G_2^2}{\Delta} + \sqrt{\frac{G_2^2 \corrupt}{ \Delta}} \,, 
    \end{align*}
    where we have also used the fact that $\lambda \leq 1$. Using the definitions of $G_1$ and $G_2$, we can conclude that
    \begin{align*}
        R_T &\leq 
        18 e^2 \frac{\capc(\Theta) \log(eT) \log (NT) }{\Delta} + 3 e \sqrt{\frac{2 \capc(\Theta) \log(eT) \log (NT)  \corrupt}{ \Delta}} +  2 \log N \\
        &\hspace{24em}+ 6 e\sqrt{2 \capc(\Theta) \log(eT) \log (NT)} 
        \\
        &\leq 36 e^2 \frac{\capc(\Theta) \log(eT) \log (NT) }{\Delta} + 3 e \sqrt{\frac{2 \capc(\Theta) \log(eT) \log (NT)  \corrupt}{ \Delta}} +  2 \log N + 4 \Delta\,,
    \end{align*}
    where we used that $6 e\sqrt{2 \capc(\Theta) \log(eT) \log (NT)} \leq \max\lcb{18 e^2 \frac{\capc(\Theta) \log(eT) \log (NT) }{\Delta}, 4\Delta}$.
\end{proof}
The appeal of this theorem is that it shows that the simple and fundamental \textsc{Exp4} algorithm can 
obtain logarithmic regret in stochastic environments, always scaling with the policy set capacity. Notice that we require the uniqueness of the optimal policy to achieve this, though similar assumptions are common in BOBW works adopting the so-called self-bounding technique applied in the last proof \citep{wei2018more,zimmert2021tsallis,ito-bobw-graphs,bobw-blackbox}.
There is, however, small room for improvement in terms of the dependence of both bounds on the time horizon. 
Compared to \Cref{thm:chi}, the adversarial bound shown above is worse off by an extra $\sqrt{\log T}$ factor. At the same time, the stochastic regime bound scales as $\log^2 T$ instead of the typical $\log T$ rate.
Arguably, these shortcomings are mild for BOBW bounds, especially considering the simplicity of the algorithm.

Nevertheless, the recent work of \cite{bobw-blackbox} offers one way for further honing these bounds.
There, a general reduction scheme is proposed, allowing the automatic synthesis of BOBW algorithms starting from traditional algorithms satisfying a certain importance-weighting (iw) stability condition.
More precisely, this condition requires that if the algorithm receives feedback in round $t$ only with probability $q_t$ (communicated at the start of the round), it achieves a gracefully degrading bound of $\E\lsb{\sqrt{c_1 \sum_{t \leq t'} 1/q_t} + c_2\max_{t \leq t'} 1/q_t}$ on the expected regret at any stopping time $t'$ with some constants $c_1$ and $c_2$. 
Let $\text{upd}_t$ be an indicator for whether the feedback is received at round $t$.
For bandits with expert advice (or contextual bandits), Lemma 10 in their paper shows that \textsc{Exp4} is iw-stable with $c_1 = \mathcal{O}(K \log N)$ and $c_2=0$ by scaling the loss estimators with $\text{upd}_t/q_t$ and using a simple adaptive learning rate. This leads to BOBW bounds depending on the number of actions $K$. 
For our setting, one can combine their analysis with that of \Cref{thm:chi} (similarly scaling the estimated shifted losses $\hat{\zeta}_t$ with $\text{upd}_t/q_t$) yielding that \textsc{Exp4} is iw-stable with $c_1 = \mathcal{O}(\capc(\Theta) \log N)$ and $c_2= \mathcal{O}(\log N)$ using the learning rate 
\[ \eta_t = \min\lcb{ \min_{s \leq t} q_s, \sqrt{ \frac{\log N}{e \capc(\Theta) \sum_{s \leq t} 1/q_s }} } \,. \]
Hence, Theorems 6 and 11 in \citep{bobw-blackbox} imply the existence of an algorithm enjoying a bound of
$
    \mathcal{O}\lrb{\sqrt{\capc(\Theta) T \log (N)} + \log (N) \log^2 (T)}
$
for the adversarial regime, and 
\begin{align*}    
    \mathcal{O}\lrb{ \frac{\capc(\Theta) \log (T) \log (N)}{\Delta} + \sqrt{\frac{\capc(\Theta) \log (T) \log (N) \corrupt}{\Delta}} + \log(N) \log(T) \log\lrb{\frac{\corrupt}{\Delta}}}
\end{align*}
for the adversarially corrupted stochastic regime.  
These bounds deliver the sought improvements at the minor cost of scaling the trailing terms in both bounds with $\log T$ factors. On the downside, achieving these bounds requires an arguably laborious and contrived combination of \textsc{Exp4} with two meta-algorithms. 

In the stochastic regime, competing results in the literature mainly include a bound of order $(1+d_\Chisqr(\Theta))^2 \log T \log N / \Delta$ in \citep[Corollary 2]{stochastic-experts},\footnote{To be precise, \cite{stochastic-experts} consider a stochastic version of the bandits with expert advice problem where an expert's recommendation is a function of an i.i.d. context. There, the diameter $d_\Chisqr$ is defined with respect to the conditional chi-squared divergence integrated over the context distribution.}
and a bound of order $\sqrt{(1+d_\Chisqr(\Theta)) T \log (NT)}$ in \citep[Theorem 2]{papini2019optimistic}. 
These bounds fall short of this section's results, primarily considering their dependence on the structure of $\Theta$. Another notable result, though incomparable, is the constant (time-independent) regret bound in \citep[Theorem 5.2]{metelli2021policy}, which nonetheless requires $d_\Chisqr(\Theta)$ to be finite.

\section{Lower Bounds} \label{sec:lower}
In this section, we complement the regret bounds provided thus far by proving lower bounds for the families of policy sets described in \Cref{sec:examples}.
More precisely, for a given policy set, we aim to prove a lower bound for the minimax regret $\inf_\pi \sup_{(\ell_t)_{t}} R_T$, where $\pi$ is the player's strategy. 
To this end, we will consider a class of environments, each identified by a vector $\mu \in [0,1]^\noutcomes$ such that, for an outcome $\oc$, the loss $\ell_t(\oc)$ at every round $t$ is drawn from a Bernoulli distribution with mean $\mu(\oc)$ in an i.i.d. manner.
For $t \leq T$, 
recall that $\his_t \coloneqq (\vartheta_s, \Oc_s, \ell_s(\Oc_s))_{s=1}^t$ denotes the interaction history till the end of round $t$.
The player's strategy $\pi$ can be represented as a sequence of probability kernels $\{\pi_t\}_{t=1}^T$, each mapping the history so far to a distribution over the policies such that $\vartheta_t$ is sampled from $\pi_t(\cdot \mid \his_{t-1})$. 
Hence, under environment $\mu$, it holds 
that $\pr(\vartheta_t = \cdot \mid \his_{t-1}) = \pi_t(\cdot \mid \his_{t-1})$, $\pr(\Oc_t = \cdot \mid \his_{t-1}, \vartheta_t) = \vartheta_t(\cdot)$, and $\pr(\ell_t(\Oc_t) = \cdot \mid \his_{t-1}, \vartheta_t, \Oc_t) = p_{\mu,\Oc_t}(\cdot)$,
where $p_{\mu,\oc}$ is the loss distribution of outcome $\oc$ under $\mu$.
Consequently, each environment $\mu$ (coupled with the player's strategy) induces a probability distribution $P_\mu$ on $\his_T$ such that
\begin{equation*} \label{eq:env-dist-def}
    P_\mu\brb{ (\xi_1,\oc_1,l_1,\dots,\xi_T,\oc_T,l_T)} = \prod\nolimits_{t=1}^T \pi_t(\xi_t \mid \xi_1,\oc_1,l_1,\dots,\xi_{t-1},\oc_{t-1},l_{t-1}) \xi_t(\oc_t) p_{\mu,\oc_t}(l_t) \,,
\end{equation*}
for any $(\xi_1,\oc_1,l_1,\dots,\xi_T,\oc_T,l_T) \in (\Theta \times \aspace \times \{0,1\})^T$.
For an environment $\mu$, we define the stochastic regret as:
\begin{equation} \label{eq:low:storeg}
    \storeg(\mu) \coloneqq \max_{\theta^* \in \Theta} \mathbb{E}_\mu  \sum_{t=1}^T \sum_{\oc \in \aspace} (\vartheta_t(\oc) - \theta^*(\oc)) \mu(\oc) \,,
\end{equation}
where the subscript in $\mathbb{E}_\mu$ emphasizes the dependence on $P_\mu$.
For any $\mu$, $\storeg(\mu)$ is a lower bound for $\sup_{(\ell_t)_t} R_T$.
Thus, to prove a lower bound on the minimax regret, it is sufficient to prove a lower bound on $\sup_\mu \storeg(\mu)$ that holds for any strategy of the player. 
In the sequel, we will make use of the following lemma, which provides an expression for the KL-divergence between the probability distributions induced by two environments. 
\begin{lemma} \label{lem:kl-decomp}
For a fixed player's strategy, policy set, and time horizon, any two environments $\mu$ and $\mu'$ satisfy 
$
    \kl{P_\mu}{P_{\mu'}} = \summ_{\theta} N_\mu(\theta; T) \summ_\oc \theta(\oc) \bklber{\mu(\oc)}{\mu'(\oc)} 
$,
where $N_\mu(\theta; T) \coloneqq \E_\mu\sum_{t=1}^T\mathbb{I}\{\vartheta_t=\theta\}$, and $\klber{a}{b}$ is the KL-divergence between two Bernoulli distributions with means $a$ and $b$.
\end{lemma} 
\begin{proof}
    Using the chain rule of the KL-divergence, one can obtain that
    \begin{equation*}
        \kl{P_\mu}{P_{\mu'}} = \summ_t \E_{\mu}\kl{p_{\mu,\oc_t}}{p_{\mu',\oc_t}}\,.
    \end{equation*}
    We then use the tower rule and the linearity of expectation to conclude the proof:
    \begin{align*}
        \summ_t \E_{\mu}\kl{p_{\mu,\oc_t}}{p_{\mu',\oc_t}} &= \summ_t \E_{\mu}\bsb{\E_{\mu}\bsb{\kl{p_{\mu,\oc_t}}{p_{\mu',\oc_t}} \mid \vartheta_t}} \\
        &= \summ_t \E_{\mu} \summ_\oc \vartheta_t(\oc) \kl{p_{\mu,\oc}}{p_{\mu',\oc}} \\
        &=\summ_{\theta} N_\mu(\theta; T) \summ_\oc \theta(\oc) \kl{p_{\mu,\oc}}{p_{\mu',\oc}} \,.
    \end{align*}
\end{proof}

\subsection{The Two Policies Case}
The following theorem provides a lower bound for the two policies case examined in \Cref{sec:examples:two}. The proof mostly follows the needle-in-a-haystack technique of \cite{adversarial}. The key to obtaining this result is a careful choice of the mean loss for each outcome. This choice leads to a lower bound in terms of the Hellinger squared distance between the two policies, which is then related to the capacity via \eqref{eq:capacity-hellinger}. This shows that the bound of \Cref{thm:chi} is order-wise unimprovable for this case. 
\begin{theorem} \label{thm:lower:two}
        Assume that $\Theta=\{\theta_1,\theta_2\}$. 
        Then, for any algorithm and $T \geq \frac{1}{8\log(4/3)H^2(\theta_1,\theta_2)}$, there exists a sequence of losses such that
        $
        R_T \geq \frac{1}{13\sqrt{2}}\sqrt{\capc(\theta_1,\theta_2)T} \,.$
\end{theorem}
\begin{proof}
    We will consider two environments $\mu_1$ and $\mu_2$ such that for an outcome $\oc$, we choose 
    \[ 
    \mu_1(\oc) \coloneqq \frac{1}{2} - \gap \frac{\sqrt{\theta_1(\oc)} - \sqrt{\theta_2(\oc)}}{\sqrt{\theta_1(\oc)} + \sqrt{\theta_2(\oc)}} 
    \qquad\text{and}\qquad
    \mu_2(\oc) \coloneqq \frac{1}{2} - \gap \frac{\sqrt{\theta_2(\oc)} - \sqrt{\theta_1(\oc)}}{\sqrt{\theta_1(\oc)} + \sqrt{\theta_2(\oc)}}
    \,,
    \]
    where $0 \leq \gap \leq \frac{1}{4}$ is to be tuned later.
    We posit that $\sqrt{\theta_1(\oc)} + \sqrt{\theta_2(\oc)}$ is always positive by assuming, without loss of generality, that each outcome is in the support of at least one policy.
    Additionally, let $\mu_0$ be an environment such that $\mu_0(\oc) \coloneqq {1}/{2}$ for any outcome $\oc$. Note that $\theta_1$ ($\theta_2$) is the optimal policy in $\mu_1$ ($\mu_2$). Indeed, 
    \begin{align*}
        \summ_{\oc} (\theta_2(\oc) - \theta_1(\oc)) \mu_1(\oc) 
        = \gap \summ_{\oc}  \brb{\sqrt{\theta_1(\oc)} - \sqrt{\theta_2(\oc)}}^2 
        = 2 \gap H^2(\theta_1,\theta_2) > 0 \,.
    \end{align*}
    Symmetrically, we have that $\summ_{\oc} (\theta_1(\oc) - \theta_2(\oc)) \mu_2(\oc) =  2 \gap H^2(\theta_1,\theta_2)$.
    Hence, it holds that 
        \begin{align}
            \storeg(\mu_1) &= \E_{\mu_1}\summ_t \I\{\vartheta_t = \theta_2\} \summ_{\oc} (\theta_2(\oc) - \theta_1(\oc)) \mu_1(\oc) \nonumber\\
            &= 2 \gap H^2(\theta_1,\theta_2) (T - N_{\mu_1}(\theta_1; T)) \nonumber\\ \label{eq:low:two:regret-kl-form}
            &\geq 2 \gap H^2(\theta_1,\theta_2) \bbrb{T - N_{\mu_0}(\theta_1; T) - T \sqrt{\frac{1}{2}\bkl{P_{\mu_0}}{P_{\mu_1}}}} \,,
        \end{align}
        where the inequality follows by using that $N_{\mu_1}(\theta_1; T) - N_{\mu_0}(\theta_1; T) \leq T \tv(P_{\mu_0},P_{\mu_1})$ followed by an application of Pinsker's inequality. Note that for $\gap \leq 1/4$ and $c\coloneqq8\log(4/3)$,
        \begin{align*}
        \klber{\mu_0(\oc)}{\mu_1(\oc)} &= \bbklber{\frac{1}{2}}{\frac{1}{2}-\gap \frac{\sqrt{\theta_1(\oc)} - \sqrt{\theta_2(\oc)}}{\sqrt{\theta_1(\oc)} + \sqrt{\theta_2(\oc)}}}
        \leq c \gap^2 \left(\frac{\sqrt{\theta_1(\oc)}- \sqrt{\theta_2(\oc)}}{\sqrt{\theta_1(\oc)} + \sqrt{\theta_2(\oc)}}\right)^2 \,.
        \end{align*}
        We also have that
        \begin{align*}
            \summ_{\oc} \theta_1(\oc) \left(\frac{\sqrt{\theta_1(\oc)}- \sqrt{\theta_2(\oc)}}{\sqrt{\theta_1(\oc)} + \sqrt{\theta_2(\oc)}}\right)^2 
            &\leq \summ_{\oc}  \brb{\sqrt{\theta_1(\oc)}- \sqrt{\theta_2(\oc)}}^2
            = 2H^2(\theta_1,\theta_2)\,,
        \end{align*}
        with an analogous inequality holding for $\theta_2$. Combining these observation with Lemma~\ref{lem:kl-decomp} gets us that 
        \begin{align*}
             &\bkl{P_{\mu_0}}{P_{\mu_1}} \\&\qquad= N_{\mu_0}(\theta_1; T) \summ_{\oc} \theta_1(\oc) \bklber{\mu_0(\oc)}{\mu_1(\oc)} 
             + N_{\mu_0}(\theta_2; T) \summ_{\oc} \theta_2(\oc) \bklber{\mu_0(\oc)}{\mu_1(\oc)}\\
             &\qquad\leq 2 c \gap^2 H^2(\theta_1,\theta_2) (N_{\mu_0}(\theta_1; T)+N_{\mu_0}(\theta_2; T)) 
             = 2 c \gap^2 H^2(\theta_1,\theta_2)T \,.
        \end{align*}
        Plugging back into \eqref{eq:low:two:regret-kl-form} yields that
        \begin{equation*}
            \storeg(\mu_1) \geq 2 \gap H^2(\theta_1,\theta_2) \bbrb{T - N_{\mu_0}(\theta_1; T) - T \gap \sqrt{c H^2(\theta_1,\theta_2)T}} \,.
        \end{equation*}
        An analogous bound can be similarly shown to hold for environment $\mu_2$. Hence, we can proceed by arguing that
            \begin{align*}
            \sup_{\mu} \storeg(\mu) &\geq \frac{1}{2}( \storeg(\mu_1)+\storeg(\mu_2))
            \geq
            \gap H^2(\theta_1,\theta_2) T \left(1 - 2\gap \sqrt{c H^2(\theta_1,\theta_2)T}\right) \,.
        \end{align*}
        The theorem then follows by setting $\gap \coloneqq \frac{1}{4\sqrt{c H^2(\theta_1,\theta_2)T}}$ 
        and using that (see \Cref{sec:examples:two}) $2 H^2(\theta_1,\theta_2) \geq \capc(\theta_1,\theta_2)$.
        Note that the stated condition on $T$ ensures that $\gap \leq 1/4$.
\end{proof}

\subsection{\texorpdfstring{$\epsilon$}{Epsilon}-Greedy Policies}
Next, we prove a lower bound for the $\epsilon$-greedy case discussed in \Cref{sec:examples:epsilon}. Recall that for this case, $\capc(\Theta) = \epsilon^2 (N-1)$; hence, the following lower bound matches the bound of \Cref{thm:chi} up to a logarithmic factor. Further, we can conclude from this result that for any $g \in [0,N-1]$, there exists a policy set $\Theta$ with $|\Theta| = N$ and $\capc(\Theta) = g$ for which one has to incur regret of order at least $\sqrt{\capc(\Theta) T}$.

\begin{theorem} \label{thm:lower:epsilon}
    Assume that the policy set conforms to the $\epsilon$-greedy structure. Then, for $T \geq \frac{N}{4\log(4/3)}$ and any algorithm, there exists a sequence of losses such that
    $
    R_T \geq \frac{1}{18}\epsilon\sqrt{NT} \,.
    $
\end{theorem}
\begin{proof}
    We will consider $N$ environments $\{\mu_{\theta}\}_{\theta \in \Theta}$ such that for environment $\mu_{\theta}$ and outcome $\oc$, $\mu_{\theta}(\oc) \coloneqq {1}/{2} - \gap\mathbb{I}\{\oc=\oc_\theta\}$, where $0 \leq \gap \leq {1}/{4}$ is to be tuned later. Additionally, let $\mu_0$ be an environment such that $\mu_0(\oc) \coloneqq {1}/{2}$ for any outcome $\oc$. Notice that $\theta$ is the optimal policy in environment $\mu_{\theta}$. In particular, for $\theta' \in \Theta \setminus \{\theta\}$, we have that
    \begin{equation*}
        \summ_{\oc} (\theta'(\oc) - \theta(\oc)) \mu_\theta(\oc) = \gap  (\theta(\oc_\theta) - \theta'(\oc_\theta)) = \gap \epsilon \,.
    \end{equation*}
    Thus, 
    \begin{align} 
    \label{eq:low:epsilon:regret-kl-form}
        \storeg(\mu_\theta) &= \gap \epsilon (T - N_{\mu_\theta}(\theta; T)) 
        \geq \gap \epsilon \bbrb{T - N_{\mu_0}(\theta; T) - T \sqrt{\frac{1}{2}\bkl{P_{\mu_0}}{P_{\mu_\theta}}}} \,,
    \end{align}
    where the inequality follows by using that $N_{\mu_\theta}(\theta; T) - N_{\mu_0}(\theta; T) \leq T \tv(P_{\mu_0},P_{\mu_\theta})$ followed by an application of Pinsker's inequality. Starting from Lemma \ref{lem:kl-decomp}, we have that 
    \begin{align*}
         \bkl{P_{\mu_0}}{P_{\mu_\theta}} &= \summ_{\theta'} N_{\mu_0}(\theta'; T) \summ_{\oc} \theta'(\oc) \bklber{\mu_0(\oc)}{\mu_\theta(\oc)} \\
         &=\summ_{\theta'} N_{\mu_0}(\theta'; T) \theta'(\oc_\theta) \bbklber{\frac{1}{2}}{\frac{1}{2}-\gap} \\
         &\leq c \gap^2 \summ_{\theta'} N_{\mu_0}(\theta'; T) \theta'(\oc_\theta) 
         = c \gap^2 \left(\frac{1-\epsilon}{N} T + \epsilon N_{\mu_0}(\theta; T)\right) \,,
    \end{align*}
    where the inequality holds for $\gap \leq {1}/{4}$ with $c \coloneqq 8\log({4}/{3})$. Plugging this result back into \eqref{eq:low:epsilon:regret-kl-form} allows us to conclude that
    \begin{align*}
        \sup_{\mu} \storeg(\mu) 
        &\geq \frac{1}{N}\summ_\theta \storeg(\mu_\theta)\\
        &\geq \gap \epsilon \left(T - \frac{T}{N} - T \sqrt{\frac{c}{2} \gap^2 \left(\frac{1-\epsilon}{N} T +  \epsilon \frac{T}{N} \right)}\right) 
        \geq \gap \epsilon T \left(\frac{1}{2} - \gap \sqrt{\frac{c}{2} \frac{T}{N}}\right) \,,
    \end{align*}
    where the second inequality uses the concavity of the square root, and the third holds since $N\geq2$. The theorem then follows by setting $\gap \coloneqq \frac{1}{4} \sqrt{\frac{2N}{cT}}$ and verifying that the stated condition on $T$ ensures that $\gap \leq {1}/{4}$.
\end{proof}

\subsection{The Multitask Bandits Structure}
In this section, we prove a lower bound for a certain structure belonging to the family of $M$-supported uniform policies described in \Cref{sec:examples:uniform}. 
Let $M$ be a positive integer 
such that $q \coloneqq {\noutcomes}/{M}$ is an integer greater than or equal to $2$. 
In this structure, the outcomes are divided into $M$ sections, and each policy is a uniform distribution supported over $M$ outcomes such that its support contains an outcome from each section. Assuming that the policy set contains all such policies, we have that $N = ({\noutcomes}/{M})^M$. For this particular structure, we will index the outcomes according to the section they belong to and their order therein: $\aspace = \{\oc_{i,j}: i\in[M], j\in[q]\}$. With this notation, we can describe the policy set as
\begin{equation*}
    \Theta = \Bigg\{ \theta \in \mathcal{U}_{\noutcomes,M}: \forall i \in [M], \sum_{j=1}^q \theta(\oc_{i,j}) = \frac{1}{M} \Bigg\} \,,
\end{equation*}
where $\mathcal{U}_{\noutcomes,M}$ is the set of all $M$-supported uniform distributions over $\noutcomes$ outcomes.
Seeing the outcomes in one section as arms in a bandit game, this problem is, in a sense, equivalent to playing $M$ bandit games simultaneously, with the choice of policy at each round dictating an arm choice for each game. The distinction is that only the loss incurred in a single randomly sampled game is observed,
while the player nonetheless aims at minimizing their regret averaged over the $M$ games.
This type of structure
(albeit with a different type of feedback)
is commonly used to prove lower bounds for combinatorial bandits \citep[see, e.g.,][]{audibert2014regret}.
The following theorem, proved in \Cref{app:multi}, provides a lower bound for our setting.

\begin{restatable}{theorem}{lowermulti}\label{thm:lower:multi}
Suppose the policy set conforms to the multi-task structure. 
Then, for any algorithm and $T \geq \frac{\noutcomes}{4\log(4/3)}$, there exists a sequence of losses such that
$
    R_T \geq \frac{1}{18} \sqrt{\noutcomes T} \,.
$
\end{restatable}

Recall from \Cref{sec:examples:uniform} that for $M$-supported uniform policies, $\capc(\Theta)={\noutcomes}/{M}-1$. Also, note that $M = \log(N) / \log({\noutcomes}/{M})$. Thus, the bound given by the theorem is of order $\sqrt{\capc(\Theta) T \log(N) / \log( \capc(\Theta) + 1)}$. The distinguishing value of this lower bound is that it shows that the logarithmic dependence on the number of policies in the bound of \Cref{thm:chi} is not entirely spurious and that it becomes increasingly tight as $C(\Theta)$ decreases. This result can be seen as an analogue for our setting of the $\sqrt{K T \log (N)/\log (K))}$ lower bound proved by \cite{expertslowerbound} for the problem of bandits with expert advice, noting that the construction of their bound relies on a (time-varying) sequence of deterministic expert recommendations. 

\section{An Impossibility Result for Linear Bandits} \label{sec:linear}
In this section, we establish a separation between the mediator feedback model and the linear bandit model in terms of achievable regret.
With $\Theta \subset \Delta_\noutcomes$ as the action set, we consider a linear bandit problem where upon choosing a policy $\vartheta_t$ in round $t$, the learner directly observes $\ell_t(\vartheta_t) = \summ_\oc \vartheta_t(\oc) \ell_t(\oc)$.
This is in contrast to the setting considered thus far, where the learner observes $(\Oc_t, \ell_t(\Oc_t))$ with $\Oc_t$ sampled from the distribution of $\vartheta_t$.
The notion of regret we aim to minimize in the linear bandit variant remains the same as before. 
The main message conveyed by this section's results is that obtaining information regarding individual outcomes is crucial for achieving regret guarantees
that reflect the affinity of the policies' distributions.

Concretely, we will consider once again the $\epsilon$-greedy structure described in \Cref{sec:examples:epsilon}.
Similar to the previous section,
we will rely on a class of environments, 
each identified by a vector $\mu \in [0,1]^\noutcomes$.
However, we will adopt a different scheme for generating the losses for the outcomes.
For every round $t$, let $Z_t$ be a random variable drawn in an i.i.d. manner from a normal distribution $\mathcal{N}(0, \sigma^2)$, with some $\sigma > 0$.
Accordingly, for each outcome $\oc$, we set $\ell_t(\oc) \coloneqq \mu(\oc) + Z_t$.
Hence, the losses assigned to the outcomes in a given round are correlated, 
see \citep{cohen2017combinatorial} for a similar approach in the combinatorial bandit problem. 
It follows then that $\ell_t(\vartheta_t) = \langle \vartheta_t, \mu \rangle + Z_t$. 
In this section, we let $P_\mu$ denote the distribution induced by $\mu$ (and the player's strategy) over the interaction history in this variant, namely 
$(\vartheta_1,\ell_1(\vartheta_1),\dots,\vartheta_T,\ell_T(\vartheta_T))$.
We will again study the stochastic regret $\storeg(\mu)$, still defined as in \eqref{eq:low:storeg}. 

For the $\epsilon$-greedy decision set, Theorem~22.1 in \citep{lattimore2020bandit} implies the existence of an algorithm enjoying a bound of order $\sqrt{NT\log(NT)}$ on the stochastic regret (when $\sigma=1$), recalling that the members of $\Theta$ in this case are $N$-dimensional vectors.
While for the adversarial regret, an upper bound of order $\sqrt{N T \log(N)}$ is achievable, see \citep{bubeck2012towards}.
The following two results show that, up to factors logarithmic in $N$ and $T$, the cited bounds are unimprovable, no matter the value of $\epsilon$. 
Hence, the growing similarity between the actions---or the shrinking diameter of $\Theta$---as $\epsilon$ approaches $0$ cannot be exploited.
This is in sharp contrast to the mediator feedback setting, where \Cref{thm:chi,thm:lower:epsilon} establish the minimax regret to be essentially of order $\epsilon\sqrt{NT}$ for this family.

\begin{proposition} \label{thm:low:linear}
    Assume that the policy set conforms to the $\epsilon$-greedy structure with $\epsilon > 0$. 
    Then, for the class of linear bandit environments described above (with any given $\sigma > 0$), 
    it holds for any algorithm and $T \geq \frac{\sigma^2 N}{\epsilon^2}$ that
    $
    \sup_{\mu} \storeg(\mu) \geq \frac{\sigma}{8} \sqrt{NT} \,.
    $
\end{proposition}
\begin{proof}
We will consider $N$ environments $\{\mu_{\theta}\}_{\theta \in \Theta}$ such that for environment $\mu_{\theta}$ and outcome $\oc$, we set
\[\mu_\theta(\oc) \coloneqq \frac{1}{2} + \gap \Brb{\frac{1-\epsilon}{N} - \mathbb{I}\{\oc=\oc_\theta\}}\,,\]
where $0 \leq \gap \leq {1}/{2}$ is to be tuned later.
Thus, under environment $\mu_\theta$, we have that
\begin{align} \label{eq:low:linear:policy-loss}
    \ell_t(\vartheta_t) = \summ_\oc \vartheta_t(\oc) \ell_t(\oc) = \frac{1}{2} + \gap \Brb{\frac{1-\epsilon}{N} - \vartheta_t(\oc_\theta)} + Z_t = \frac{1}{2} - \gap \epsilon \mathbb{I}\{\vartheta_t=\theta\} + Z_t \,.
\end{align}
Additionally, let $\mu_0$ be an environment such that $\mu_0(\oc) = {1}/{2}$ for any outcome $\oc$, implying that $\ell_t(\vartheta_t) = 1/2 + Z_t$. 
Notice that $\theta$ is the optimal policy in environment $\mu_{\theta}$. 
In particular, for $\theta' \in \Theta \setminus \{\theta\}$, we have that
    \[
        \summ_\oc (\theta'(\oc) - \theta(\oc)) \mu_\theta(\oc) 
        = \gap  (\theta(\oc_\theta) - \theta'(\oc_\theta)) 
        = \gap \epsilon \,.
    \]
Hence, it holds that
\begin{align*}
    \storeg(\mu_\theta) &= \gap \epsilon (T - N_{\mu_\theta}(\theta; T)) 
    \geq \gap \epsilon \bbrb{T - N_{\mu_0}(\theta; T) - T \sqrt{\frac{1}{2}\bkl{P_{\mu_0}}{P_{\mu_\theta}}}} \,,
\end{align*}
where the inequality follows by using that $N_{\mu_\theta}(\theta; T) - N_{\mu_0}(\theta; T) \leq T \tv(P_{\mu_0},P_{\mu_\theta})$ followed by an application of Pinsker's inequality. 
Combining the observation in \eqref{eq:low:linear:policy-loss} with standard results (see, for example, Exercise 15.8 (b) and Exercise 14.7 in \citep{lattimore2020bandit}), we can express the KL-divergence term as follows:
\begin{align*}
    \bkl{P_{\mu_0}}{P_{\mu_\theta}} &= \summ_{\theta'} N_{\mu_0}(\theta'; T) \bkl{\mathcal{N}(1/2,\sigma^2)}{\mathcal{N}(1/2 - \gap \epsilon \mathbb{I}\{\theta'=\theta\},\sigma^2)} \\ 
    &= N_{\mu_0}(\theta; T) \bkl{\mathcal{N}(1/2,\sigma^2)}{\mathcal{N}(1/2 - \gap \epsilon ,\sigma^2)} 
    = \frac{\gap^2 \epsilon^2}{2 \sigma^2} N_{\mu_0}(\theta; T)\,.
\end{align*}
Consequently, we get that
\begin{align*}
    \sup_{\mu} \storeg(\mu) 
    &\geq \frac{1}{N}\summ_\theta \storeg(\mu_\theta)
    \geq \gap \epsilon \left(T - \frac{T}{N} - \frac{\gap \epsilon}{2 \sigma} T \sqrt{\frac{T}{N}}\right)
    \geq \gap \epsilon T \left(\frac{1}{2} - \frac{\gap \epsilon}{2 \sigma} \sqrt{\frac{T}{N}}\right) \,,
\end{align*}
where the second inequality holds by the concavity of the square root, and the third since $N\geq2$. 
The proposition then follows by choosing $\gap \coloneqq \frac{\sigma}{2\epsilon}\sqrt{\frac{N}{T}}$. 
Note that the stated condition on $T$ ensures that $\gap \leq {1}/{2}$.
\end{proof}

As the construction of this lower bound relied on normally distributed (hence unbounded) losses, a lower bound in the adversarial setting is not immediately implied. 
Instead, the following theorem (proved in \Cref{app:lower-bounded}) provides the sought bound at the cost of an extra $1/\sqrt{\log(T)}$ factor by combining Proposition \ref{thm:low:linear} with a simple truncation argument due to \cite{cohen2017combinatorial}.
Notice that the resulting bound can be made arbitrarily larger than the mediator feedback guarantee of \Cref{thm:chi} by picking a small enough $\epsilon$ and a suitably long horizon.
\begin{restatable}{theorem}{lowerlinearbounded} \label{thm:low:linear-bounded}
    Assume that the policy set conforms to the $\epsilon$-greedy structure with $\epsilon > 0$. 
    Then, under linear bandit feedback,
    we have that for any algorithm and $T \geq \frac{N}{8\epsilon^2}$, there exists a sequence of losses (bounded in $[0,1]$) such that
    $
    R_T \geq \frac{1}{64\sqrt{2\log(16 T)}} \sqrt{NT} \,.
    $
\end{restatable}
 
\section{The Full-Information Case} \label{sec:full-info}
In this last section, we briefly examine a full-information variant of the problem, where the entire loss map $(\ell_t(\oc))_{\oc \in \aspace}$ is observed at every round. One can see this as a variant of the classical prediction with expert advice problem \citep{cesa1997use}, with the outcomes representing the actions (or experts).
The distinction is that the learner can only sample from a mixture of the distributions of the policies in $\Theta$.
Moreover, the learner competes with the best policy, aiming to minimize the same notion of regret as before.
We show in the following theorem that a simple strategy 
enjoys a regret guarantee depending on the more standard notion of capacity based on the KL-divergence. 
Precisely, if $\vartheta$ and $\Oc$ are two random variables taking values respectively over $\Theta$ and $\aspace$ such that $\pr_{\Oc|\vartheta=\theta}(\oc) = \theta(\oc)$ for any $\theta \in \Theta$ and $\oc \in \aspace$. 
Then, we define the KL-capacity of the policy set as:
\begin{align*}
    \capckl(\Theta) \coloneqq \capckl(\pr_{\Oc|\vartheta})
    = \max_{\tau \in \mP_\Theta} \summ_\theta \tau(\theta) \bkl{\theta} {\textstyle{\summ_{\theta'}} \tau(\theta')\theta'}\,.
\end{align*}
Via \citep[Corollary 5.6]{itbook}, $\capckl(\Theta)$ can alternatively be interpreted as the ``radius'' of $\Theta$ in terms of the KL-divergence:
\begin{equation} \label{eq:full:radius}
    \capckl(\Theta) = \min_{\centroid \in \Delta_\noutcomes} \max_{\theta \in \Theta} \kl{\theta} {\centroid} \,.
\end{equation}

The idea of \Cref{alg:osmd} is to run an Online Mirror Descent (OMD) algorithm directly on the outcome space restricted to the convex hull of the policy set, henceforth denoted as $\co(\Theta)$. The key to obtaining the following result is a tailored choice of the initial distribution that utilizes the interpretation in \eqref{eq:full:radius}.

\begin{algorithm} [t]
    \caption{OMD on the Convex Hull of the Policies Under Full-Information}
    \label{alg:osmd}
    \begin{algorithmic}[1]
        \State \textbf{Input:} learning rate $\eta$, initial distribution $\centroid^* \in \co(\Theta): \centroid^*(\oc)>0 \: \forall \oc \in \aspace$
        \State \textbf{Initialize:} $u_1=\centroid^*$
        \For{$t=1,\dotsc,T$}
            \State Pick distribution $p_t \in \mP_\Theta$ such that $\sum_{\theta} p_t(\theta) \theta = u_t$
            \State Draw $\vartheta_t \sim p_t$ 
            \State Observe the loss vector $\ell_t$ 
            \State Set $u_{t+1} = \argmin_{u \in \co(\Theta)} \eta \langle u, \ell_t \rangle + \kl{u}{u_t}$
        \EndFor
    \end{algorithmic}
\end{algorithm}

\begin{theorem} \label{thm:full-info}
Algorithm \ref{alg:osmd} run with
$
\centroid^* \in \argmin_{\centroid\in\Delta_{\noutcomes}} \max_{\theta \in \Theta} \kl{\theta} {\centroid}$ and $\eta=\sqrt{\frac{2\capckl(\Theta)}{T}}
$
satisfies
$
       R_T \leq \sqrt{2\capckl(\Theta)T}
$.
\end{theorem}

\begin{proof}
    Firstly, we note that setting $u_1 = \centroid^*$ is a valid choice since the minimum value of $\max_{\theta \in \Theta} \kl{\theta} {\centroid}$ in $\centroid \in \Delta_\noutcomes$ can only be attained in $\mathrm{co}(\Theta)$; see Theorem 11.6.1 in \cite{cover2006}.
    Let $\theta^* \in \argmin_{\theta \in \Theta} \E \summ_t \ell_t(\theta)$.
    As $\ell_t$ and $\vartheta_t$ are independent given the events up to the end round $t-1$, we have that
    $R_T = \E \sum\nolimits_t \langle \vartheta_t - \theta^*, \ell_t \rangle = \E \sum\nolimits_t \langle u_t - \theta^*, \ell_t \rangle$.
    A standard regret bound for OMD with the negative entropy regularizer (see, e.g., Lemma 6.16 and the proof of Theorem 10.2 in \citealp{orabona2023modern}) allows us to conclude that
    \begin{equation*} \label{cvxhull-omd}
        \sum\nolimits_t \langle u_t - \theta^*, \ell_t \rangle \leq \frac{\kl{\theta^*} {\centroid^*}}{\eta} + \frac{\eta}{2} \sum\nolimits_t\summ_\oc u_t(\oc) \ell^2_t(\oc) \leq \frac{\kl{\theta^*} {\centroid^*}}{\eta} + \frac{\eta}{2} T \leq \frac{\capckl(\Theta)}{\eta} + \frac{\eta}{2} T \,,
    \end{equation*}
    where the last step follows from \eqref{eq:full:radius} and the definition of $\centroid^*$. The theorem then follows after plugging in the value of $\eta$.
\end{proof}

We remark that $\capckl(\Theta) \leq \min\{\log N, \log \noutcomes\}$, see \Cref{sec:info-theory}. 
In particular, the first bound is attained when the policies have non-overlapping supports, while the second is attained when each outcome is matched with a policy entirely concentrated on that outcome. 
Notice that these two cases essentially reduce the problem to a standard prediction with expert advice problem on the policy and outcome spaces, respectively, where the bound of \Cref{thm:full-info} matches the minimax regret up to constants \citep{cesa1997use}.
Beyond these extreme cases, the bound improves as the capacity of the policy set---or its information radius---shrinks. 

\section{Conclusions and Future Directions}
In this paper, we have focused on the mediator feedback framework and studied to what extent the structure of the problem can be exploited by the learner. In particular, we have introduced the policy set capacity as a measure of the effective size (or complexity) of the policy set. For various setups, we have derived new and improved regret bounds for \textsc{Exp4} featuring the capacity.
Further, the lower bounds we provided establish the capacity as a fundamental indicator of the difficulty of the problem.

One direction for improvement is providing bounds that hold with high probability rather than in expectation, noting that prior works on bandits with expert advice (or contextual bandits) such as \citep{beygelzimer2011contextual} and \citep{neu15hp} obtained high probability bounds only of order $\sqrt{KT \log N}$.
Another direction is proving data-dependent (or small loss) bounds, following again previous works on bandits with expert advice \citep[e.g.,][]{allen-zhu18b}. 
Concerning the stochastic regime, improving the dependence of the bounds on the sub-optimality gaps (beyond the crude scaling with the smallest gap) is yet another interesting problem. 
Finally, for cases when the policy set is very large, achieving similar regret guarantees via more computationally efficient strategies is a worthy direction. 


\acks{
The financial support of the FAIR (Future Artificial Intelligence Research) project, funded by the NextGenerationEU program within the PNRR-PE-AI scheme, is gratefully acknowledged. KE and NCB also acknowledge the support of the MUR PRIN grant 2022EKNE5K (Learning in Markets and Society), funded by the NextGenerationEU program within the PNRR scheme and of the EU Horizon CL4-2022-HUMAN-02 research and innovation action under grant agreement 101120237, project ELIAS.
}



\appendix

\section{Proof of Theorem \ref{thm:lower:multi}} \label{app:multi}
\lowermulti*
\begin{proof} 
    In the following, we will overload the notation and denote by $\oc_{i,\theta}$---which belongs to $\{\oc_{i,j}\}_{j=1}^q$---the outcome chosen by policy $\theta$ in section $i$ (i.e. we have that $\theta(\oc_{i,\theta}) = {1}/{M}$).
    For each policy $\theta$, we construct an environment $\mu_\theta$ such that for any outcome $\oc$, 
    $\mu_{\theta}(\oc) \coloneqq {1}/{2} - \gap\mathbb{I}\{\oc \in \text{Supp}(\theta)\}\,,$
    where $0 \leq \gap \leq {1}/{4}$ is to be tuned later. Moreover, we will also use the following variations of each environment. For $i \in [M]$, let $\mu_\theta^{-i}$ be an environment such that for any outcome $\oc$, 
    \[ \mu_\theta^{-i}(\oc) \coloneqq 
    \begin{cases}
    \frac{1}{2} \,, & \text{if } \oc \in \{\oc_{i,j}\}_{j=1}^q\\
    \mu_{\theta}(\oc) \,, & \text{otherwise.}
    \end{cases}\]
    In words, $\mu_\theta^{-i}$ is identical to $\mu_\theta$ everywhere except in game $i$, where all outcomes are assigned a mean loss of $1/2$. 
    For any policy $\theta$, we have that
    \begin{align*}
         \storeg(\mu_\theta) 
         &= \gap \E_{\mu_\theta}\sum_{t=1}^T \sum_{i=1}^M (\theta(\oc_{i,\theta}) - \vartheta_t(\oc_{i,\theta}))  \\
         &= \frac{\gap}{M} \E_{\mu_\theta}\sum_{t=1}^T \sum_{i=1}^M (1 - \mathbb{I}\{\oc_{i,\vartheta_t} = \oc_{i,\theta}\}) \\
         &= \frac{\gap}{M} \sum_{i=1}^M  (T - N_{\mu_\theta}(i,\theta;T)) \,,
    \end{align*}
    where for an environment $\mu$, a policy $\theta$, and a section $i \in [M]$, we define $N_{\mu}(i,\theta; T) \coloneqq \E_{\mu}\sum_{t=1}^T  \mathbb{I}\{\oc_{i,\vartheta_t} = \oc_{i,\theta}\}$. 
    In words, this counts the expected number of times (under $\mu$) that the chosen policy agrees with $\theta$ in section $i$. 
    Next, we use that for any $i \in [M]$, $N_{\mu_\theta}(i,\theta;T) - N_{\mu^{-i}_\theta}(i,\theta;T) \leq T \tv\brb{P_{\mu^{-i}_\theta},P_{\mu_\theta}}$ together with Pinsker's inequality to get that
    \begin{equation} \label{low:multi:reg-kl}
        \storeg(\mu_\theta) \geq \frac{\gap}{M} \sum_{i=1}^M  \bigg(T - N_{\mu^{-i}_\theta}(i,\theta;T) - T\sqrt{\frac{1}{2}\bkl{P_{\mu^{-i}_\theta}}{P_{\mu_\theta}}}\bigg) \,.
    \end{equation}
    For bounding the KL-divergence term, we start from Lemma \ref{lem:kl-decomp}:
    \begin{align*}
         \bkl{P_{\mu^{-i}_\theta}}{P_{\mu_\theta}} 
         &= \sum_{\theta' \in \Theta} N_{\mu^{-i}_\theta}(\theta';T) \sum_{\oc\in\aspace} \theta'(\oc) \bklber{\mu^{-i}_\theta(\oc)}{\mu_\theta(\oc)} \\
         &= \sum_{\theta' \in \Theta} N_{\mu^{-i}_\theta}(\theta';T)  \theta'(\oc_{i,\theta}) \bklber{\mu^{-i}_\theta(\oc_{i,\theta})}{\mu_\theta(\oc_{i,\theta})} \\
         &= \frac{1}{M}\sum_{\theta' \in \Theta} \mathbb{I}\{\oc_{i,\theta'} = \oc_{i,\theta}\}  N_{\mu^{-i}_\theta}(\theta';T)  \bbklber{\frac{1}{2}}{\frac{1}{2}-\gap} \\
         &\leq \frac{c \gap^2}{M}\sum_{\theta' \in \Theta} \mathbb{I}\{\oc_{i,\theta'} = \oc_{i,\theta}\} N_{\mu^{-i}_\theta}(\theta';T)  \\
         &= \frac{c \gap^2}{M}   \E_{\mu^{-i}_\theta}\sum_{t=1}^T  \mathbb{I}\{\oc_{i,\vartheta_t} = \oc_{i,\theta}\}
         = \frac{c \gap^2}{M} N_{\mu^{-i}_\theta}(i,\theta;T) \,,  
    \end{align*}
    where the second equality holds since $\oc_{i,\theta}$ is the only outcome that does not have the same mean loss in the two environments, and the inequality holds for $\gap \leq {1}/{4}$ with $c \coloneqq 8\log({4}/{3})$. Plugging back into \eqref{low:multi:reg-kl}, we get that
    \begin{align} \label{low:multi:reg-single}
        \storeg(\mu_\theta)  \geq \frac{\gap}{M} \sum_{i=1}^M  \bigg(T - N_{\mu^{-i}_\theta}(i,\theta;T) - T\gap\sqrt{\frac{c}{2 M} N_{\mu^{-i}_\theta}(i,\theta;T)}\bigg) \,.
    \end{align}
    For what follows, we introduce an extra bit of notation. 
    For each $i \in [M]$, we let $\sim_i$ denote an equivalence relation on the policy set such that for $\theta$, $\theta' \in \Theta$,
    \[
        \theta \sim_i \theta' \iff \forall s \in [M]\backslash\{i\}, \oc_{s, \theta} = \oc_{s, \theta'} \,.
    \]
    In words, two policies are equivalent according to $\sim_i$ if they agree everywhere outside of section $i$. Denote the set of all equivalence classes of $\sim_i$ by $\Theta / \sim_i$, which contains $q^{M-1}$ classes, each containing $q$ policies corresponding to the possible outcome choices in section $i$. For any $Y \in \Theta / \sim_i$, notice that if $\theta, \theta' \in Y$, then $\mu_{\theta}^{-i}$ and $\mu_{\theta'}^{-i}$ denote the same environment, 
    which will be referred to in the sequel as $\mu_{Y}^{-i}$.
    Now, notice that for any section $i$,
    \begin{align*}
        \sum_{\theta \in \Theta} N_{\mu^{-i}_\theta}(i,\theta;T) 
        &= \sum_{Y \in \Theta/\sim_i}
        \sum_{\theta \in Y} N_{\mu^{-i}_Y}(i,\theta;T) 
        = \sum_{Y \in \Theta/\sim_i}
         \E_{\mu^{-i}_Y}\sum_{t=1}^T \underbrace{\sum_{\theta \in Y} \mathbb{I}\{\oc_{i,\vartheta_t} = \oc_{i,\theta}\}}_{=1} 
         = q^{M-1} T \,,
    \end{align*}
    whereas
    $
        \sum_{\theta \in \Theta} \sqrt{  N_{\mu^{-i}_\theta}(i,\theta;T)} \leq \sqrt{\sum_{\theta \in \Theta} 1^2} \sqrt{\sum_{\theta \in \Theta}  N_{\mu^{-i}_\theta}(i,\theta;T)}  
        = q^{M} \sqrt{\frac{T}{q}} \,.
    $
    Hence, we conclude that
    \begin{align*} 
        \sup_\mu \storeg(\mu) 
        &\geq \frac{1}{|\Theta|} \sum_{\theta \in \Theta}  \storeg(\mu_\theta) \\
        &\geq \frac{1}{|\Theta|} \sum_{\theta \in \Theta}  \frac{\gap}{M} \sum_{i=1}^M  \lrb{T - N_{\mu^{-i}_\theta}(i,\theta;T) - T\gap\sqrt{\frac{c}{2 M} N_{\mu^{-i}_\theta}(i,\theta;T)}} \\
        &\geq \frac{\gap}{M} \sum_{i=1}^M   \lrb{ T - \frac{1}{|\Theta|}  q^{M} T\lrb{ \frac{1}{q} + \gap\sqrt{\frac{c T}{2 q M} } }} \\
        &= \gap T   \bigg (1 -   \frac{1}{q} - \gap\sqrt{\frac{c T}{2 \noutcomes} } \bigg) 
        \stackrel{q\geq2}{\geq} \gap T\left(\frac{1}{2} - \gap\sqrt{\frac{cT}{2\noutcomes}} \right) \,. 
    \end{align*}
    Plugging $\gap \coloneqq \frac{1}{4} \sqrt{\frac{2\noutcomes}{cT}}$ into the previous display proves the theorem after observing that $\frac{1}{16}\sqrt{\frac{2}{c}} \geq \frac{1}{18}$. Lastly, notice that the condition imposed on $T$ ensures that indeed $\gap \leq {1}/{4}$.
\end{proof}

\section{Proof of Theorem \ref{thm:low:linear-bounded}} \label{app:lower-bounded}
\lowerlinearbounded*
\begin{proof}
    Building on the result of Proposition \ref{thm:low:linear}, we follow the technique used in the proof of Theorem 5 in  \citep{cohen2017combinatorial}.
    For the class of linear bandit environments specified in \Cref{sec:linear}, we define
    \begin{gather*}
        \hat{R}_T(\mu) \coloneqq \max_{\theta^* \in \Theta}   \sum_{t=1}^T \sum_{\oc \in \aspace} (\vartheta_t(\oc) - \theta^*(\oc)) (\mu(\oc) + Z_t) = \max_{\theta^* \in \Theta}   \sum_{t=1}^T \sum_{\oc \in \aspace} (\vartheta_t(\oc) - \theta^*(\oc)) \mu(\oc)
    \\
        \tilde{R}_T(\mu) \coloneqq \max_{\theta^* \in \Theta}   \sum_{t=1}^T \sum_{\oc \in \aspace} (\vartheta_t(\oc) - \theta^*(\oc)) \clip(\mu(\oc) + Z_t) \,,
    \end{gather*}
    where $\clip(a) \coloneqq \max\{\min\{a,1\},0\}$. 
    Notice that $\E \hat{R}_T(\mu) \geq \storeg(\mu)$, and that $\sup_{(\ell_t)_t} R_T \geq \sup_\mu \E \tilde{R}_T(\mu)$ considering sequences of losses $(\ell_t)_t$ bounded in $[0,1]$.
    We also define the event $A_\mu \coloneqq \{\forall t\in [T] ,\: \oc \in \aspace \colon  \clip(\mu(\oc) + Z_t) = \mu(\oc) + Z_t\}$.
    We will consider again the environments $\{\mu_{\theta}\}_{\theta \in \Theta}$ used in the proof of Proposition \ref{thm:low:linear}, recalling that $\mu_\theta(\oc) \coloneqq {1}/{2} + \gap \lrb{(1-\epsilon)/{N} - \mathbb{I}\{\oc=\oc_\theta\}}$ for some $0 \leq \gap \leq 1/2$.
    For any $\theta$, we have that
    \begin{align} \label{eq:low:linear:unclipped-regret}
        \E \hat{R}_T(\mu_\theta) &= \E\bsb{\hat{R}_T(\mu_\theta) \I\{A_{\mu_\theta}\}} + \E\bsb{\hat{R}_T(\mu_\theta) \I\{A^c_{\mu_\theta}\}} 
        \leq \E\bsb{\tilde{R}_T(\mu_\theta)} + \gap \epsilon T \pr\brb{A^c_{\mu_\theta}} \,,
    \end{align}
    where we have used the fact that $\tilde{R}_T({\mu_\theta})$ and $\hat{R}_T({\mu_\theta})$ are identical when $A_{\mu_\theta}$ occurs, and that $\hat{R}_T({\mu_\theta})$ is uniformly bounded by $\gap \epsilon T$ (see proof of Proposition \ref{thm:low:linear}).
    Assuming we enforce that $\gap \leq 1/4$, the event $\{\clip(\mu(\oc) + Z_t) \neq \mu(\oc) + Z_t\}$ cannot hold for any outcome unless $|Z_t| > 1/4$.
    Hence, using a union bound and the fact that $Z_t \sim \mathcal{N}(0, \sigma^2)$, we get that
    \begin{align*}
        \pr\brb{A^c_{\mu_\theta}} 
        &\leq \summ_t \pr\brb{ |Z_t| > 1/4 } \leq 2 T \exp\lrb{\frac{-(1/4)^2}{2 \sigma^2}} \,.
    \end{align*}
    Combining this with \eqref{eq:low:linear:unclipped-regret} and the fact that $\E \hat{R}_T(\mu_\theta) \geq \storeg(\mu_\theta)$ allows us to conclude that
    \begin{align*}
        \sup_{(\ell_t)_t} R_T \geq 
        \sup_\mu \E\tilde{R}_T(\mu) \geq \frac{1}{N}\summ_\theta \E\tilde{R}_T(\mu_\theta) 
        \geq \frac{1}{N}\summ_\theta \storeg(\mu_\theta) -  2 \gap \epsilon T^2 \exp\lrb{\frac{-(1/4)^2}{2 \sigma^2}} \,.
    \end{align*}
    Setting $\gap \coloneqq \frac{\sigma}{2\epsilon}\sqrt{\frac{N}{T}}$, we obtain from the proof of Proposition \ref{thm:low:linear} that $\frac{1}{N}\summ_\theta \storeg(\mu_\theta) \geq \frac{\sigma}{8} \sqrt{NT}$. Hence, choosing $\sigma \coloneqq 1/(4\sqrt{2 \log(16T)})$ entails that the required bound. Notice that the condition $T \geq \frac{N}{8 \epsilon^2}$ suffices to ensure that $\gap \leq 1/4$.
\end{proof}

\vskip 0.2in
\bibliography{refs}

\begin{thebibliography}{48}
\providecommand{\natexlab}[1]{#1}
\providecommand{\url}[1]{\texttt{#1}}
\expandafter\ifx\csname urlstyle\endcsname\relax
  \providecommand{\doi}[1]{doi: #1}\else
  \providecommand{\doi}{doi: \begingroup \urlstyle{rm}\Url}\fi

\bibitem[Ali and Silvey(1966)]{Ali1966AGC}
S.~M. Ali and S.~D. Silvey.
\newblock A general class of coefficients of divergence of one distribution from another.
\newblock \emph{Journal of the Royal Statistical Society: Series B (Methodological)}, 28\penalty0 (1):\penalty0 131--142, 1966.

\bibitem[Allen-Zhu et~al.(2018)Allen-Zhu, Bubeck, and Li]{allen-zhu18b}
Z.~Allen-Zhu, S.~Bubeck, and Y.~Li.
\newblock Make the minority great again: First-order regret bound for contextual bandits.
\newblock In \emph{Proceedings of the 35th International Conference on Machine Learning}, volume~80 of \emph{Proceedings of Machine Learning Research}, pages 186--194. PMLR, 2018.

\bibitem[Alon et~al.(2015)Alon, Cesa-Bianchi, Dekel, and Koren]{Alon15}
N.~Alon, N.~Cesa-Bianchi, O.~Dekel, and T.~Koren.
\newblock Online learning with feedback graphs: Beyond bandits.
\newblock In \emph{Proceedings of The 28th Conference on Learning Theory}, volume~40 of \emph{Proceedings of Machine Learning Research}, pages 23--35. PMLR, 2015.

\bibitem[Arumugam and Van~Roy(2021)]{arumugam2021deciding}
D.~Arumugam and B.~Van~Roy.
\newblock Deciding what to learn: A rate-distortion approach.
\newblock In \emph{Proceedings of the 38th International Conference on Machine Learning}, volume 139 of \emph{Proceedings of Machine Learning Research}, pages 373--382. PMLR, 2021.

\bibitem[Audibert and Bubeck(2009)]{audibert2009minimax}
J.~Audibert and S.~Bubeck.
\newblock Minimax policies for adversarial and stochastic bandits.
\newblock In \emph{Proceedings of the 22nd Conference on Learning Theory}, 2009.

\bibitem[Audibert et~al.(2014)Audibert, Bubeck, and Lugosi]{audibert2014regret}
J.~Audibert, S.~Bubeck, and G.~Lugosi.
\newblock Regret in online combinatorial optimization.
\newblock \emph{Mathematics of Operations Research}, 39\penalty0 (1):\penalty0 31--45, 2014.

\bibitem[Auer et~al.(1995)Auer, Cesa-Bianchi, Freund, and Schapire]{casino}
P.~Auer, N.~Cesa-Bianchi, Y.~Freund, and R.~E. Schapire.
\newblock Gambling in a rigged casino: The adversarial multi-armed bandit problem.
\newblock In \emph{Proceedings of IEEE 36th annual foundations of computer science}, pages 322--331. IEEE, 1995.

\bibitem[Auer et~al.(2002{\natexlab{a}})Auer, Cesa-Bianchi, Freund, and Schapire]{adversarial}
P.~Auer, N.~Cesa-Bianchi, Y.~Freund, and R.~E. Schapire.
\newblock The nonstochastic multiarmed bandit problem.
\newblock \emph{SIAM Journal on Computing}, 32\penalty0 (1):\penalty0 48--77, 2002{\natexlab{a}}.

\bibitem[Auer et~al.(2002{\natexlab{b}})Auer, Cesa-Bianchi, and Gentile]{AUER-adaptive}
P.~Auer, N.~Cesa-Bianchi, and C.~Gentile.
\newblock Adaptive and self-confident on-line learning algorithms.
\newblock \emph{Journal of Computer and System Sciences}, 64\penalty0 (1):\penalty0 48--75, 2002{\natexlab{b}}.

\bibitem[Beygelzimer et~al.(2011)Beygelzimer, Langford, Li, Reyzin, and Schapire]{beygelzimer2011contextual}
A.~Beygelzimer, J.~Langford, L.~Li, L.~Reyzin, and R.~Schapire.
\newblock Contextual bandit algorithms with supervised learning guarantees.
\newblock In \emph{Proceedings of the Fourteenth International Conference on Artificial Intelligence and Statistics}, volume~15 of \emph{Proceedings of Machine Learning Research}, pages 19--26. PMLR, 2011.

\bibitem[Bubeck and Slivkins(2012)]{bubeck12bobw}
S.~Bubeck and A.~Slivkins.
\newblock The best of both worlds: Stochastic and adversarial bandits.
\newblock In \emph{Proceedings of the 25th Annual Conference on Learning Theory}, volume~23 of \emph{Proceedings of Machine Learning Research}, pages 42.1--42.23. PMLR, 2012.

\bibitem[Bubeck et~al.(2012)Bubeck, Cesa-Bianchi, and Kakade]{bubeck2012towards}
S.~Bubeck, N.~Cesa-Bianchi, and S.~M. Kakade.
\newblock Towards minimax policies for online linear optimization with bandit feedback.
\newblock In \emph{Proceedings of the 25th Annual Conference on Learning Theory}, volume~23 of \emph{Proceedings of Machine Learning Research}, pages 41.1--41.14. PMLR, 2012.

\bibitem[Cesa-Bianchi et~al.(1997)Cesa-Bianchi, Freund, Haussler, Helmbold, Schapire, and Warmuth]{cesa1997use}
N.~Cesa-Bianchi, Y.~Freund, D.~Haussler, D.~P. Helmbold, R.~E. Schapire, and M.~K. Warmuth.
\newblock How to use expert advice.
\newblock \emph{Journal of the ACM (JACM)}, 44\penalty0 (3):\penalty0 427--485, 1997.

\bibitem[Cohen et~al.(2017)Cohen, Hazan, and Koren]{cohen2017combinatorial}
A.~Cohen, T.~Hazan, and T.~Koren.
\newblock Tight bounds for bandit combinatorial optimization.
\newblock In \emph{Proceedings of the 2017 Conference on Learning Theory}, volume~65 of \emph{Proceedings of Machine Learning Research}, pages 629--642. PMLR, 2017.

\bibitem[Cover and Thomas(2006)]{cover2006}
T.~M. Cover and J.~A. Thomas.
\newblock \emph{Elements of information theory {(2.} ed.)}.
\newblock Wiley, 2006.

\bibitem[Csisz{\'a}r(1967)]{Csiszar}
I.~Csisz{\'a}r.
\newblock On information-type measure of difference of probability distributions and indirect observations.
\newblock \emph{Studia Scientiarum Mathematicarum Hungarica}, 2:\penalty0 299--318, 1967.

\bibitem[Dann et~al.(2023)Dann, Wei, and Zimmert]{bobw-blackbox}
C.~Dann, C.~Wei, and J.~Zimmert.
\newblock A blackbox approach to best of both worlds in bandits and beyond.
\newblock In \emph{Proceedings of Thirty Sixth Conference on Learning Theory}, volume 195 of \emph{Proceedings of Machine Learning Research}, pages 5503--5570. PMLR, 12--15 Jul 2023.

\bibitem[Eldowa et~al.(2023)Eldowa, Cesa-Bianchi, Metelli, and Restelli]{itw-paper}
K.~Eldowa, N.~Cesa-Bianchi, A.~M. Metelli, and M.~Restelli.
\newblock Information-theoretic regret bounds for bandits with fixed expert advice.
\newblock In \emph{2023 IEEE Information Theory Workshop (ITW)}, pages 30--35, 2023.

\bibitem[Guntuboyina(2011)]{minimax-risk}
A.~Guntuboyina.
\newblock Lower bounds for the minimax risk using $ f $-divergences, and applications.
\newblock \emph{IEEE Transactions on Information Theory}, 57\penalty0 (4):\penalty0 2386--2399, 2011.

\bibitem[Ito et~al.(2022)Ito, Tsuchiya, and Honda]{ito-bobw-graphs}
S.~Ito, T.~Tsuchiya, and J.~Honda.
\newblock Nearly optimal best-of-both-worlds algorithms for online learning with feedback graphs.
\newblock In \emph{Advances in Neural Information Processing Systems}, volume~35, pages 28631--28643. Curran Associates, Inc., 2022.

\bibitem[Krishnamurthy et~al.(2020)Krishnamurthy, Langford, Slivkins, and Zhang]{krishnamurthy2020contextual}
A.~Krishnamurthy, J.~Langford, A.~Slivkins, and C.~Zhang.
\newblock Contextual bandits with continuous actions: Smoothing, zooming, and adapting.
\newblock \emph{The Journal of Machine Learning Research}, 21\penalty0 (1):\penalty0 5402--5446, 2020.

\bibitem[Lattimore and Szepesv{\'a}ri(2020)]{lattimore2020bandit}
T.~Lattimore and C.~Szepesv{\'a}ri.
\newblock \emph{Bandit algorithms}.
\newblock Cambridge University Press, 2020.

\bibitem[Majzoubi et~al.(2020)Majzoubi, Zhang, Chari, Krishnamurthy, Langford, and Slivkins]{majzoubi2020efficient}
M.~Majzoubi, C.~Zhang, R.~Chari, A.~Krishnamurthy, J.~Langford, and A.~Slivkins.
\newblock Efficient contextual bandits with continuous actions.
\newblock In \emph{Advances in Neural Information Processing Systems}, volume~33, pages 349--360. Curran Associates, Inc., 2020.

\bibitem[Makur and Zheng(2020)]{makur2020comparison}
A.~Makur and L.~Zheng.
\newblock Comparison of contraction coefficients for f-divergences.
\newblock \emph{Problems of Information Transmission}, 56:\penalty0 103--156, 2020.

\bibitem[McMahan and Streeter(2009)]{McMahanS09}
H.~B. McMahan and M.~J. Streeter.
\newblock Tighter bounds for multi-armed bandits with expert advice.
\newblock In \emph{Proceedings of the 22nd Conference on Learning Theory}, 2009.

\bibitem[McMahan and Streeter(2010)]{McMahanS10}
H.~B. McMahan and M.~J. Streeter.
\newblock Adaptive bound optimization for online convex optimization.
\newblock In \emph{Proceedings of the 23rd Conference on Learning Theory}, pages 244--256. Omnipress, 2010.

\bibitem[Metelli et~al.(2021)Metelli, Papini, D'Oro, and Restelli]{metelli2021policy}
A.~M. Metelli, M.~Papini, P.~D'Oro, and M.~Restelli.
\newblock Policy optimization as online learning with mediator feedback.
\newblock In \emph{Proceedings of the AAAI Conference on Artificial Intelligence}, volume~35, pages 8958--8966, 2021.

\bibitem[Misra et~al.(2019)Misra, Schwartz, and Abernethy]{misra2019dynamic}
K.~Misra, E.~M. Schwartz, and J.~Abernethy.
\newblock Dynamic online pricing with incomplete information using multiarmed bandit experiments.
\newblock \emph{Marketing Science}, 38\penalty0 (2):\penalty0 226--252, 2019.

\bibitem[Neu(2015{\natexlab{a}})]{neu15}
G.~Neu.
\newblock First-order regret bounds for combinatorial semi-bandits.
\newblock In \emph{Proceedings of The 28th Conference on Learning Theory}, volume~40 of \emph{Proceedings of Machine Learning Research}, pages 1360--1375. PMLR, 2015{\natexlab{a}}.

\bibitem[Neu(2015{\natexlab{b}})]{neu15hp}
G.~Neu.
\newblock Explore no more: Improved high-probability regret bounds for non-stochastic bandits.
\newblock In \emph{Advances in Neural Information Processing Systems}, volume~28. Curran Associates, Inc., 2015{\natexlab{b}}.

\bibitem[Orabona(2023)]{orabona2023modern}
F.~Orabona.
\newblock A modern introduction to online learning.
\newblock \emph{arXiv preprint arXiv:1912.13213}, 2023.

\bibitem[Papini et~al.(2019)Papini, Metelli, Lupo, and Restelli]{papini2019optimistic}
M.~Papini, A.~M. Metelli, L.~Lupo, and M.~Restelli.
\newblock Optimistic policy optimization via multiple importance sampling.
\newblock In \emph{Proceedings of the 36th International Conference on Machine Learning}, volume~97 of \emph{Proceedings of Machine Learning Research}, pages 4989--4999. PMLR, 2019.

\bibitem[Poiani et~al.(2023)Poiani, Metelli, and Restelli]{poiani2023pure}
R.~Poiani, A.~M. Metelli, and M.~Restelli.
\newblock Pure exploration under mediators’ feedback.
\newblock In \emph{NeurIPS 2023 Workshop on Adaptive Experimental Design and Active Learning in the Real World}, 2023.

\bibitem[Polyanskiy and Wu(2023)]{itbook}
Y.~Polyanskiy and Y.~Wu.
\newblock \emph{Information Theory: From Coding to Learning}.
\newblock Cambridge University Press, 2023.
\newblock (draft).

\bibitem[Raginsky(2016)]{raginsky2016strong}
M.~Raginsky.
\newblock Strong data processing inequalities and {$\Phi$-S}obolev inequalities for discrete channels.
\newblock \emph{IEEE Transactions on Information Theory}, 62\penalty0 (6):\penalty0 3355--3389, 2016.

\bibitem[Reddy et~al.(2023)Reddy, Karthik, Karamchandani, and Nair]{reddy2023}
K.~S. Reddy, P.~N. Karthik, N.~Karamchandani, and J.~Nair.
\newblock Best arm identification in bandits with limited precision sampling.
\newblock In \emph{2023 IEEE International Symposium on Information Theory (ISIT)}, pages 1466--1471, 2023.

\bibitem[Russo and Van~Roy(2016)]{russo2016information}
D.~Russo and B.~Van~Roy.
\newblock An information-theoretic analysis of thompson sampling.
\newblock \emph{The Journal of Machine Learning Research}, 17\penalty0 (1):\penalty0 2442--2471, 2016.

\bibitem[Russo and Van~Roy(2022)]{russo2022satisficing}
D.~Russo and B.~Van~Roy.
\newblock Satisficing in time-sensitive bandit learning.
\newblock \emph{Mathematics of Operations Research}, 47\penalty0 (4):\penalty0 2815--2839, 2022.

\bibitem[Sason and Verd{\'u}(2016)]{sason2016f}
I.~Sason and S.~Verd{\'u}.
\newblock $ f $-divergence inequalities.
\newblock \emph{IEEE Transactions on Information Theory}, 62\penalty0 (11):\penalty0 5973--6006, 2016.

\bibitem[Schwartz et~al.(2017)Schwartz, Bradlow, and Fader]{schwartz2017customer}
E.~M. Schwartz, E.~T. Bradlow, and P.~S. Fader.
\newblock Customer acquisition via display advertising using multi-armed bandit experiments.
\newblock \emph{Marketing Science}, 36\penalty0 (4):\penalty0 500--522, 2017.

\bibitem[Seldin and Lugosi(2016)]{expertslowerbound}
Y.~Seldin and G.~Lugosi.
\newblock A lower bound for multi-armed bandits with expert advice.
\newblock In \emph{The 13th European Workshop on Reinforcement Learning (EWRL)}, 2016.

\bibitem[Seldin et~al.(2011)Seldin, Auer, Shawe-taylor, Ortner, and Laviolette]{seldin2011pac}
Y.~Seldin, P.~Auer, J.~Shawe-taylor, R.~Ortner, and F.~Laviolette.
\newblock Pac-bayesian analysis of contextual bandits.
\newblock In \emph{Advances in Neural Information Processing Systems}, volume~24, pages 1683--1691. Curran Associates, Inc., 2011.

\bibitem[Sen et~al.(2018)Sen, Shanmugam, and Shakkottai]{stochastic-experts}
R.~Sen, K.~Shanmugam, and S.~Shakkottai.
\newblock Contextual bandits with stochastic experts.
\newblock In \emph{Proceedings of the Twenty-First International Conference on Artificial Intelligence and Statistics}, volume~84 of \emph{Proceedings of Machine Learning Research}, pages 852--861. PMLR, 2018.

\bibitem[Shalev-Shwartz et~al.(2012)]{shalev2012online}
S.~Shalev-Shwartz et~al.
\newblock Online learning and online convex optimization.
\newblock \emph{Foundations and Trends{\textregistered} in Machine Learning}, 4\penalty0 (2):\penalty0 107--194, 2012.

\bibitem[Topsoe(2000)]{topsoe}
F.~Topsoe.
\newblock Some inequalities for information divergence and related measures of discrimination.
\newblock \emph{IEEE Transactions on information theory}, 46\penalty0 (4):\penalty0 1602--1609, 2000.

\bibitem[Wei and Luo(2018)]{wei2018more}
C.~Wei and H.~Luo.
\newblock More adaptive algorithms for adversarial bandits.
\newblock In \emph{Proceedings of the 31st Conference On Learning Theory}, volume~75 of \emph{Proceedings of Machine Learning Research}, pages 1263--1291. PMLR, 2018.

\bibitem[Zhu and Mineiro(2022)]{zhu2022contextual}
Y.~Zhu and P.~Mineiro.
\newblock Contextual bandits with smooth regret: Efficient learning in continuous action spaces.
\newblock In \emph{Proceedings of the 39th International Conference on Machine Learning}, volume 162 of \emph{Proceedings of Machine Learning Research}, pages 27574--27590. PMLR, 2022.

\bibitem[Zimmert and Seldin(2021)]{zimmert2021tsallis}
J.~Zimmert and Y.~Seldin.
\newblock Tsallis-inf: An optimal algorithm for stochastic and adversarial bandits.
\newblock \emph{The Journal of Machine Learning Research}, 22\penalty0 (1):\penalty0 1310--1358, 2021.

\end{thebibliography}

\end{document}